\documentclass[12pt]{article}
\usepackage{localtr}
\usepackage{mathptmx}
\usepackage{amsthm}
\usepackage{amsmath}
\usepackage{amsfonts}
\usepackage{latexsym}
\usepackage{listings}
\usepackage[x11names,svgnames]{xcolor}
\usepackage{tikz}
\usetikzlibrary{automata,positioning}
\usetikzlibrary{shapes}
\usepackage{adjustbox}
\usepackage[compatibility=false]{caption}
\usepackage{subcaption}
\usepackage{algpseudocode}
\usepackage{algorithm}
\usepackage{cmll}
\usepackage{todonotes}
\usepackage{soul}
\usepackage{url}
\usepackage{hyperref}
\usepackage{thmtools}
\usepackage{forest}
\usepackage{vmargin}
\setpapersize{USletter}
\setmarginsrb{25.4mm}{25.4mm}{25.4mm}{17.4mm}{0mm}{0mm}{0mm}{8mm}

\newcommand{\comment}[1]{{}}

\newcommand{\when}{\ \mbox{${:\!\!-\;}$}}

\newtheorem{Lem}{Lemma}

\newtheorem{Thm}[Lem]{Theorem}

\newtheorem{Def}{Definition}

\newcommand{\lsem}{\mbox{$[\![$}}
\newcommand{\rsem}{\mbox{$]\!]$}}
\newcommand{\sem}[1]{\mbox{$\lsem #1 \rsem$}}

\newcommand{\id}[1]{\mbox{\it #1\/}}

\newcommand{\kw}[1]{\mbox{\tt #1}}

\def\p@enumiii{\theenumi(\theenumii)}

\newcommand\annotate[1]%
{\textcolor{blue}{$^\maltese$}%
\-\marginpar[\raggedleft\scriptsize \textcolor{red}{#1}]%
{\raggedright\scriptsize \textcolor{red}{#1}}}

\newcommand{\myparagraph}[1]{\paragraph{\textbf{#1}}}

\newcommand\email[1]{{\normalfont\rmfamily
  \itshape\textup{(}e-mail: \textup{\texttt{#1})}}}

\date{}

\begin{document}
\label{firstpage}
\title{Inference in Probabilistic Logic Programs using
  Lifted~Explanations\thanks{This work was supported in part by NSF
    Grant IIS-1447549.}}
\author{Arun Nampally and C.\ R.\ Ramakrishnan\\
Computer Science Department,\\
Stony Brook University, Stony Brook, NY 11794.\\
\email{anampally@cs.stonybrook.edu, cram@cs.stonybrook.edu}
}

\maketitle

\lstset{%
  language=Prolog,
  basicstyle=\ttfamily,
  commentstyle=\rmfamily\it\color{DarkBlue},
  columns=fullflexible,
  numbers=left, numberstyle=\tiny, stepnumber=1, numbersep=5pt,
  firstnumber=auto,
  numberfirstline=true,
  numberblanklines=true
}

\begin{abstract}
  In this paper, we consider the problem of lifted inference in the
  context of Prism-like probabilistic logic programming languages.
  Traditional inference in such languages involves the construction of
  an explanation graph for the query and computing probabilities over
  this graph.  When evaluating queries over probabilistic logic
  programs with a large number of instances of random variables,
  traditional methods treat each instance separately.  For many
  programs and queries, we observe that explanations can be summarized
  into substantially more compact structures, which we call “lifted
  explanation graphs”.  In this paper, we define lifted explanation
  graphs and operations over them.  In contrast to existing lifted
  inference techniques, our method for constructing lifted
  explanations naturally generalizes existing methods for constructing
  explanation graphs.  To compute probability 
  of query answers, we solve recurrences generated from the lifted
  graphs.  We show examples where the use of our technique reduces the
  asymptotic complexity of inference.
\end{abstract}

\section{Introduction}\label{sec:intro}
\myparagraph{Background.}  
Probabilistic Logic Programming (PLP) provides a declarative
programming framework to specify and use combinations of logical and
statistical models.  A number of programming languages and systems
have been proposed and studied under the framework of
PLP, e.g. PRISM~\cite{sato1997prism}, Problog~\cite{de2007problog},
PITA~\cite{riguzzi2011pita} and Problog2~\cite{dries2015problog2} etc.  These
languages have similar declarative semantics based on the
\emph{distribution semantics}~\cite{sato2001parameter}.  Moreover, the
inference algorithms used in many of these systems to evaluate the
probability of query answers, 
e.g.~PRISM, Problog and PITA, are based on a
common notion of \emph{explanation graphs}.

At a high level, the inference procedure follows traditional query
evaluation over logic programs.  Outcomes of random variables, i.e.,
the \emph{probabilistic choices}, are abduced during query evaluation.
Each derivation of an answer is associated with a set of outcomes of
random variables, called its explanation, under which the answer
is supported by the derivation. \comment{ Consequently, the collection of such
sets over all derivations of an answer represents the set of all
outcomes under which the answer can be derived.}  Systems differ on how
the explanations are represented and manipulated.  Explanation
graphs in PRISM are represented using tables, and
under mutual exclusion assumption, multiple explanations are
combined by adding entries to tables.
In Problog and PITA, explanation graphs
are represented by Binary Decision Diagrams (BDDs), with probabilistic
choices mapped to propositional variables in BDDs.

\myparagraph{Driving Problem.}
Inference based on explanation graphs does not scale well to
logical/statistical models with large numbers of random processes and
variables.  Several \emph{approximate inference} techniques have been
proposed to estimate the probability of answers when exact inference
is infeasible. In general, large logical/statistical models involve
\emph{families} of independent, identically distributed (i.i.d.)
random variables.   Moreover, in many models, inference often
depends on the outcomes of random processes but not on the identities of 
random variables with the particular outcomes.
However, query-based inference methods will instantiate
each random variable and the explanation graph will represent each of
their outcomes.  Even when the graph may ultimately exhibit symmetry
with respect to random variable identities, and many parts of the
graph may be shared,  the computation that produced these graphs may
not be shared.    \emph{This paper presents
a structure for representing explanation graphs compactly
by exploiting the symmetry with respect to i.i.d random variables, and a procedure to build this structure
without enumerating each instance of a random process.}

\myparagraph{Illustration.}  We illustrate the problem and our
approach using the simple example in Figure~\ref{fig:intro-ex}, which
shows a program describing a process of tossing a number of
i.i.d. coins, and evaluating if at least two of them came up
``heads''.  The example is specified in an extension of the PRISM
language, called Px.  Explicit random processes of PRISM enables a
clearer exposition of our approach.  In PRISM and Px, a special
predicate of the form \texttt{msw($p$, $i$, $v$)} describes, given a
random process $p$ that defines a family of i.i.d. random variables,
that $v$ is the value of the $i$-th random variable in the family.
The argument $i$ of \texttt{msw} is called the \emph{instance}
argument of the predicate.  In this paper, we consider Param-Px, a
further extension of Px to define parameterized programs.  In
Param-Px, a built-in predicate, \texttt{in} is used to specify
membership; e.g. \texttt{$x$ in $s$} means $x$ is member of an
enumerable set $s$.  The size of $s$ is specified by a separate
\texttt{population} directive.

The program in Figure~\ref{fig:intro-ex} defines a family of random
variables with outcomes in \{\texttt{h}, \texttt{t}\} generated by 
\texttt{toss}.  The instances that index these random variables are drawn
from the set \texttt{coins}.  Finally, 
predicate \texttt{twoheads} is defined to hold if tosses of at least two
distinct coins come up ``heads''.  

\begin{figure}[h]
  \centering
  \small
  \begin{tabular}{c@{\extracolsep{-3em}}c@{\extracolsep{2em}}c}
    \begin{minipage}[c]{0.35\linewidth}
      \begin{adjustbox}{scale=0.68}
        \begin{lstlisting}[name=IntroExample]
% Two distinct tosses show "h"
twoheads :-
    X in coins, 
    msw(toss, X, h),
    Y in coins, 
    {X < Y},
    msw(toss, Y, h).
    
% Cardinality of coins:
:- population(coins, 100).

% Distribution parameters:
:- set_sw(toss, 
categorical([h:0.5, t:0.5])).
\end{lstlisting}
\end{adjustbox}
\end{minipage}
&
  \begin{minipage}[c]{0.33\textwidth}
    \begin{adjustbox}{scale=0.65}
      \begin{tikzpicture}[shorten >=1pt,%
          inner sep = 0.5mm, %
          node distance=2.25cm,%
          every state/.style={circle}, 
          on grid,auto,%
          initial text=,%
          ellipsis/.style={dotted},%
          bend angle=30]
          \node[state,rectangle] (s11) {$(\mathtt{toss},1)$};
          \node[state,rectangle] (s21) [below right=of s11] {$(\mathtt{toss},2)$};
          \node[state,rectangle] (s22) [below left=of s11] {$(\mathtt{toss},2)$};
          \node[state,rectangle] (s31) [below right=of s21] {$(\mathtt{toss},n-1)$};
          \node[state,rectangle] (s32) [below left=of s21] {$(\mathtt{toss},3)$};
          \node[state] (t3) [below left=of s22] {$1$};
          \node[state] (s41) [below right=of s31] {$0$};
          \node[state,rectangle] (s42) [below left=of s31] {$(\mathtt{toss},n)$};
          \node[state] (t4) [below left=of s32] {$1$};
          \node[state] (sn1) [below right=of s42] {$0$};
          \node[state] (sn2) [below left=of s42] {$1$};

          \path[->] 
          (s11) edge node {t} (s21)
          (s11) edge [swap] node {h} (s22)
          (s21) edge  [ellipsis] node {t} (s31)
          (s21) edge [swap] node {h} (s32)
          (s22) edge node {t} (s32)
          (s22) edge [swap] node {h} (t3)
          (s31) edge node {t} (s41)
          (s31) edge [swap] node {h} (s42)
          (s32) edge [ellipsis] node {t} (s42)
          (s32) edge [swap] node {h} (t4)
          (s42) edge node {t} (sn1)
          (s42) edge [swap] node {h} (sn2);
        \end{tikzpicture}
      \end{adjustbox}
   \end{minipage}
&
  \begin{minipage}[c]{0.3\textwidth}
     \begin{adjustbox}{scale=0.7}
        \begin{tikzpicture}[shorten >=1pt,%
          inner sep = 0.5mm, %
          node distance=2.25cm,%
          every state/.style={circle}, 
          on grid,auto,%
          initial text=,%
          highlight/.style={draw=red, text=red},%
          bend angle=30]

          \node[state, rectangle] (c) {$\exists \mathtt{X}. \exists \mathtt{Y}. \mathtt{X} <
            \mathtt{Y}$};
          \node[state, rectangle,rounded corners=2mm] (s11) [below = of c] {$(\mathtt{toss}, \mathtt{X})$};
          \node[state, rectangle,rounded corners=2mm] (s21) [below left=of s11] {$(\mathtt{toss}, \mathtt{Y})$};
          \node[state] (s22) [below right=of s11] {$0$};
          \node[state] (s31) [below left=of s21] {$1$};
          \node[state] (s32) [below right=of s21] {$0$};

          \path[->] 
          (s11) edge [swap] node {h} (s21)
          (s11) edge node {t} (s22)
          (s21) edge [swap] node {h} (s31)
          (s21) edge node {t} (s32);
        \end{tikzpicture}
      \end{adjustbox}
    \end{minipage}
\\
(a) Simple Px program & (b) Ground expl. Graph
    & (c) Lifted expl. Graph\\
  \end{tabular}
  \caption{Example program and ground explanation graph}
  \label{fig:intro-ex}
\end{figure}

\myparagraph{State of the Art, and Our Solution.}
Inference in PRISM, Problog and PITA follows the
structure of the derivations for a query.  Consider the program in 
Figure~\ref{fig:intro-ex}(a) and let the cardinality of the set of coins be $n$.   
The query \texttt{twoheads} will take $\Theta(n^2)$ time, since it will
construct bindings to both \texttt{X} and \texttt{Y} in the clause
defining \texttt{twoheads}.    However, the size of an explanation
graph is $\Theta(n)$; see Figure~\ref{fig:intro-ex}(b).
Computing the probability of the query over this graph will also take $\Theta(n)$
time.  

In this paper, we present a technique to construct a symbolic version
of an explanation graph, called a \emph{lifted explanation graph} that
represents instances symbolically and avoids
enumerating the instances of random processes such as \texttt{toss}.
The lifted explanation graph for query \texttt{twoheads} is shown in
Figure~\ref{fig:intro-ex}(c).   Unlike traditional explanation graphs
where nodes are specific instances of random variables, nodes in the
lifted explanation graph may be parameterized by their instance (e.g
$(\mathtt{toss}, X)$ instead of $(\texttt{toss}, 1)$). A set of constraints on
those variables, specify the allowed groundings.

Note that the graph size is independent of the size of the population.
Moreover, the graph can be constructed in time independent of the population
size as well.  Probability computation is performed by first deriving
recurrences based on the graph's structure and then solving the recurrences. The
recurrences for probability computation derived from the graph in
Fig.~\ref{fig:intro-ex}(c) are shown in Fig.~\ref{fig:intro-ex-eqn}. In the
figure, the equations with subscript 1 are derived from the root of the graph;
those with subscript 2 from the left child of the root; and where $\pi$ is the
probability that \texttt{toss} is ``\texttt{h}''.  Note that the probability of
the query, $f_1(n)$, can be computed in $\Theta(n)$ time from the recurrences.

\begin{figure}
  \centering
  \begin{minipage}{1.0\linewidth}
  \begin{adjustbox}{scale=0.84}
    \small
    \begin{tabular}{l|l}
    $\begin{array}{rl}
      f_1(n)  = & h_1(1, n)\\
      h_1(i, n) = & \left\{
        \begin{array}{ll}
          g_1(i, n)  + (1 - \widehat{f_1}) \cdot  h_1(i+1,n) & \mbox{ if } i < n\\
          g_1(i, n) & \mbox{ if } i = n
        \end{array} \right.\\
      g_1(i, n) = & \pi \cdot f_2(i, n)\\
      % P(\widehat{\phi}_{1(toss,c)}) & = & \pi\\
      \widehat{f_1} = & \pi \\
      \end{array}$
      &
    $\begin{array}{rl}
      f_2(i, n) = & \left\{
        \begin{array}{ll}
          h_2(i+1, n) & \mbox{ if } i < n\\
          0 & \mbox{ otherwise}
        \end{array} \right.\\
      h_2(j, n) = & \left\{
        \begin{array}{ll}
          g_2 + (1 -\widehat{f_2}) \cdot  h_2(j+1, n) & \mbox{ if } j < n\\
          g_2 & \mbox{ if } j = n
        \end{array}
      \right.\\
      g_2 = & \pi\\
      \widehat{f_2} = & \pi
    \end{array}$
  \end{tabular}
  \end{adjustbox}
  \end{minipage}
\caption{Recurrences for computing probabilities for Example in \protect{Fig.~\ref{fig:intro-ex}}}
\label{fig:intro-ex-eqn}
\end{figure}

These recurrences can be solved in $O(n)$ time with tabling or dynamic
programming.  Moreover,  in certain cases, it is possible to obtain a
closed form from a recurrence.  For instance, noting that $g_2$ is
independent of its parameters, we get $h_2(j,n) = 1-(1-\pi)^{n-j+1}$.

\myparagraph{Lifted explanations vs. Lifted Inference.}  Our work is a
form of \emph{lifted inference}, a set of techniques that have been
intensely studied in the context of first-order graphical models and
Markov Logic Networks~\cite{poole2003first,braz2005lifted,milch2008lifted}.  Essentially, lifted explanations
provide a way to perform lifted inference over PLPs by leveraging
their query evaluation mechanism.  Directed first-order graphical models~\cite{kisynski2010aggregation} can be
readily cast as PLPs, and our technique can be used to perform lifted
inference over such models.  Our solution, however, does not
cover techniques based on counting
elimination~\cite{braz2005lifted,milch2008lifted}.  

It should be noted that Problog2 does not construct query-specific
explanation graphs.  Instead, it uses a knowledge compilation approach
where the models of a program are represented by a
propositional boolean formula.  These formulae, in turn, are
represented in a compact standard form such as 
dDNNFs or SDDs~\cite{darwiche2001tractable,darwiche2011sdd}.  Query answers
and their probabilities are then computed using linear-time algorithms
over these structures.  

The knowledge compilation approach has been
extended to do a generalized form of lifted inference using
first-order model counting~\cite{van2011lifted}.  
This technique performs lifted inference, including inversion and
counting elimination over a large class of
first order models.  However, first order
model counting is defined only when the problem can be stated in a
first-order constrained CNF form.  Problems such as the example in
Figure~\ref{fig:intro-ex} cannot be written in that form.
To address this, a skolemization procedure
which eliminates existential quantifiers and converts to first-order CNF without
adding function symbols was proposed by
\citeN{broeck2014skolemization}.  While the knowledge compilation
approach takes a core lifted inference procedure and moves to apply it
to a class of logic programs, our approach generalizes existing
inference techniques to perform a form of lifted inference.

\myparagraph{Contributions.}  The technical contribution of this paper
is two fold. 
\begin{enumerate}
\item We define a lifted explanation structure, and operations
  over these structures (see Section~\ref{sec:liftedexp}).  We also
  give method to construct such structures during query evaluation,
  closely following the techniques used to construct explanation
  graphs.
\item We define a
  technique to compute probabilities over such structures by deriving
  and solving recurrences (see Section~\ref{sec:inference}).  We
  provide examples to illustrate the complexity gains due to our
  technique over traditional inference.
\end{enumerate}
The rest of the paper begins by defining parameterized Px
programs and their semantics (Section~\ref{sec:parampx}).
After presenting the main technical work, the paper concludes with a 
discussion of  related work.
(Section~\ref{sec:related}).

%%% ATTIC:
\comment{

The example describes a system with a
large number of independent agents, each of whom can be in one of two
states, active or inactive, denoted by ``\texttt{a}'' and
``\texttt{b}'' in the program.  The evolution of agent states is
governed by a Markov chain; at each time step, an agent may change its
state according to a given probability distribution that depends only
on the agent's current state.  The predicate \texttt{agent(X,S,T)},
denoting that the agent \texttt{X} is at state \texttt{S} at time
\texttt{T}, is defined based on a Markov chain.

In the program,
\texttt{msw(init,X,S)} denotes that \texttt{S} is the initial state of
agent \texttt{X}; the domain and distribution of \texttt{S} is
specified by \texttt{value}and \texttt{set\_sw} directives for process
\texttt{init}.  Similarly, \texttt{msw(step(S1),(X,T),S)} represents a
state transition for agent \texttt{X} from \texttt{S1} to \texttt{S}
at time \texttt{T} .  Note the use of pair \texttt{(X,T)} as the
instance variable, thereby defining a family of i.i.d. random
variables indexed by \texttt{(X,T)}, each variable associated with a
specific agent \texttt{X} at a specific time \texttt{T}.

The predicate \texttt{samestate(T)} holds whenever there are at least
two distinct agents in ``\texttt{a}'' state at time \texttt{T}.  In
the program, \texttt{population} predicate specifies the domain of
agent ids.  Note
that the identity of the agents is irrelevant to the value of
\texttt{samestate}.  However, an explanation graph for
\texttt{samestate($t$)} for a specific time instant $t$ will have
$\Theta{nt}$ nodes if there are $n$ distinct agents.   The enumeration
over agents can be avoided by ``lifting'' the explanation graphs by
parameterizing them with (quantified) variables.    

}

\comment{
\begin{figure}[h]
  \centering
\small
\begin{tabular}{cp{0.0in}cp{0.0in}c}
  \begin{minipage}[t]{2in}
    \begin{adjustbox}{scale=0.9}
\begin{lstlisting}[name=IntroExample]
% Agent X is in state S at time T.
agent(X, S, 0) :-
    msw(init, X, S).
agent(X, S, T) :-
    T > 0, T1 is T-1,
    agent(X, S1, T1),
    msw(step(S1), (X,T), S).
% Domains:
values(init, [a,b]).
values(trans(_), [a,b]).
% Distribution parameters:
set_sw(init, [0.01,0.99]).
set_sw(step(a), [0.1,0.9]).
set_sw(step(b), [0.01,0.99]).

% Two agents are in state "a"
%     simultaneously
samestate(T) :-
    instance(p1, X),
    agent(X, a, T),
    instance(p1, Y),
    X \= Y,
    agent(X, a, T).
population(p1, 100).
\end{lstlisting}
    \end{adjustbox}
  \end{minipage}
& Graph 1 & Graph 2\\
(a) Program describing a collection of agents && (b) Ground
Explanation Graph && (c) Lifted Explanation Graph
\end{tabular}
  \caption{PLP description of a collection of agents,}
  \label{fig:intro-ex}
\end{figure}
}

%%% Local Variables: 
%%% mode: latex
%%% TeX-master: "main"
%%% End: 

%\input{prelim}
\section{Parameterized Px Programs}\label{sec:parampx}
The PRISM language follows Prolog's syntax.  It adds a binary
predicate \texttt{msw} to introduce random variables into an
otherwise familiar Prolog program.  Specifically, in \texttt{msw($s$,
  $v$)}, $s$ is a ``switch'' that represents a random process which
generates a family of random variables, and $v$
is bound to the value of a variable in that family.  The domain and
distribution parameters of the switches are specified using
\texttt{value} facts and \texttt{set\_sw} directives, respectively. 
Given a switch $s$, we use $D_s$ to denote the domain of $s$, and
$\pi_s: D_s \rightarrow [0,1]$ to denote its probability distribution.
%We assume that $\sum_{v\in D_s} \pi_s(v) = 1$.

The model-theoretic distribution semantics explicitly identifies each
member of a random variable family with an \emph{instance} parameter.
In the PRISM system, the binary \texttt{msw} is interpreted
stochastically, generating a new member of the random variable family
whenever an \texttt{msw} is encountered during inference.  

\subsection{Px and Inference}\label{sec:pxandinf}
The Px language extends the PRISM language in three ways.  Firstly, the
\texttt{msw} switches in Px are ternary, with the addition of an
explicit \emph{instance} parameter.  This brings the language closer
to the formalism presented when describing PRISM's
semantics~\cite{sato2001parameter}.  Secondly, Px aims to compute the
distribution semantics with no assumptions on the structure of the
explanations.  Thirdly, in contrast to PRISM,
the switches in Px can 
be defined with a wide variety of univariate distributions, including
continuous distributions (such as Gaussian) and infinite discrete
distributions (such as Poisson).  
\comment{
Syntactically, Px's \texttt{set\_sw}
directive combines the specification of the
domain and distribution of a switch.  At the time of this writing, the Px
system supports approximate sampling-based inference in general, and
exact inference when only finite discrete distributions are used.
The semantics of programs with continuous random variables was described
in~\cite{IRR:ICLP12,islam2012inference}.  Algorithms for exact inference over programs with
Gaussian- and Gamma-distributed switches, and parameter learning over
program with Gaussian-distributed switches are known as
well~\cite{islam2012inference}. 
}
However, in this paper, we consider only programs with finite
discrete distributions.

Exact inference of Px programs with finite discrete distributions uses
explanation graphs with the following structure.
\begin{Def}[Ground Explanation Graph]\label{def:explanation-graph}
Let $S$ be the set of ground switches in a Px program $P$, and $D_s$ be
the domain of switch $s \in S$.  Let ${\cal T}$ be the set of all
ground terms over symbols in $P$.   Let ``$\prec$'' be  a total order
over $S \times {\cal T}$ such that $(s_1,t_1) \prec (s_2,t_2)$ if either $t_1 <
t_2$ or $t_1=t_2$ and $s_1 < s_2$.

A \emph{ground explanation tree}
over $P$ is a rooted tree  $\gamma$ such that:
\begin{itemize}
\item Leaves in $\gamma$ are labeled $0$ or $1$.
\item Internal nodes in $\gamma$ are labeled $(s,z)$ where $s\in S$
  is a switch, and $z$ is a ground term over symbols in $P$.
\item For node labeled $(s,z)$, there are $k$ outgoing edges to subtrees, where
  $k=|D_s|$.  Each edge is labeled with a unique $v \in D_s$. 
\item Let $(s_1,z_1), (s_2,z_2), \ldots, (s_k, z_k), c$ be the sequence of node labels in a
  root-to-leaf path in the tree, where $c\in \{0,1\}$.  Then
  $(s_i,z_i) \prec (s_j, z_j)$ if $i<j$ for all $i,j
  \leq k$.  As a corollary, node labels along any root to leaf path in the tree are unique.
\end{itemize}
An \emph{explanation graph} is a DAG representation of a ground
explanation tree.  \hfill $\Box$
\end{Def}
We use $\phi$ to denote explanation graphs.  We use
$(s,t)[v_i:\phi_i]$ to denote an explanation graph whose root is
labeled $(s,t)$, with each edge labeled $v_i$ (ranging over a suitable
index set $i$), leading to subgraph $\phi_i$.

\sloppypar
Consider a sequence of alternating node and edge labels in a
root-to-leaf path: $(s_1,z_1), v_1, (s_2,z_2), v_2, \ldots, (s_k,z_k),
v_k, c$.   Each
such path enumerates a set of random variable valuations $\{s_1[z_1] =
v_1, s_2[z_2] = v_2, \ldots, s_k[z_k] = v_k\}$.  When $c = 1$, the set
of valuations forms an explanation.   An explanation graph thus
represents a set of explanations.

Note that explanation trees and graphs resemble decision diagrams.
Indeed, explanation graphs are implemented using Binary Decision
Diagrams~\cite{bryant1992symbolic} in PITA and Problog; and Multi-Valued Decision
Diagrams~\cite{srinivasan1990algorithms} in Px.  The \emph{union} of two sets of
explanations can be seen as an ``\emph{or}'' operation over
corresponding explanation graphs.  Pair-wise union of explanations in
two sets is an ``\emph{and}'' operation over corresponding explanation graphs.

\paragraph{Inference via Program Transformation.}
Inference in Px is performed analogous to that in
PITA~\cite{riguzzi2011pita}.  Concretely, inference is done by
translating a Px program to one that explicitly constructs explanation
graphs, performing tabled evaluation of the derived program, and
computing probability of answers from the explanation graphs.  We
describe the translation for definite pure programs; programs with
built-ins and other constructs can be translated in a similar manner.

First every clause containing a disequality constraint is replaced by two
clauses using less-than constraints. Next, for every user-defined atom $A$ of
the form $p(t_1, t_2, \ldots, t_n)$, we define $\id{exp}(A,E)$ as atom
$p(t_1, t_2, \ldots, t_n, E)$ with a new predicate $p/(n+1)$, with $E$ as an
added ``explanation'' argument.  For such atoms $A$, we also define
$\id{head}(A,E)$ as atom $p'(t_1, t_2, \ldots, t_n, E)$ with a new predicate
$p'/(n+1)$.  A \emph{goal} $G$ is a conjunction of atoms, where $G = (G_1, G_2)$
for goals $G_1$ and $G_2$, or $G$ is an atom $A$.  Function $\id{exp}$ is
extended to goals such that
$\id{exp}((G_1, G_2)) = ((\id{exp}(G_1,E_1), \id{exp}(G_2,E_2)),
\kw{and}(E_1,E_2,E))$, where \texttt{and} is a predicate in the translated
program that combines two explanations using conjunction, and $E_1$ and $E_2$
are fresh variables.  Function $\id{exp}$ is also extended to \texttt{msw} atoms
such that $\id{exp}(\mathtt{msw}(p,i,v), E)$ is $\texttt{rv}(p,i,v,E)$, where
\texttt{rv} is a predicate that binds $E$ to an explanation graph with root
labeled $(p,i)$ with an edge labeled $v$ leading to a $1$ child, and all other
edges leading to $0$.

Each clause of the form $A \when G$ in a Px program is
translated to a new clause $\id{head}(A,E) \when \id{exp}(G,E)$.  For
each predicate $p/n$, we define 
$p(X_1, X_2, \ldots X_n, E)$ to be such that $E$ is the disjunction of
all $E'$ for  $p'(X_1, X_2, \ldots X_n, E')$.  As in PITA, this is done using
answer subsumption.

\paragraph{Computing Answer Probabilities.}
Probability of an answer is determined by first materializing the explanation
graph, and then computing the probability over the graph. The probability
associated with a node in the graph is computed as the sum of the products of
probabilities associated with its children and the corresponding edge
probabilities. The probability associated with an explanation graph $\varphi$,
denoted $\id{prob}(\varphi)$ is the probability associated with the root. This
can be computed in time linear in the size of the graph by using dynamic
programming or tabling.

\comment{
\begin{Def}[Explanation Probability]\label{def:answer-probability}
Let $\phi$ be an explanation graph over a Px program.  Let $\pi_s$ be
the probability distribution of switch $s$ in the program.
Then the probability of $\phi$, denoted by $\id{prob}(\phi)$ is
defined by the recurrence:
\begin{align*}
  \id{prob}(0) &= 0\\
  \id{prob}(1) &= 1\\
  \id{prob}((s,z)[v_i:\phi_i]) &= \sum_i \pi_s(v_i) \times
  \id{prob}(\phi_i)
\end{align*}
Probability can be computed in time linear in the
size of an explanation graph.  \hfill $\Box$
\end{Def}

Once the explanation graph of an answer is materialized, the
probability of an answer is determined by computing the probability of
its explanation.  
\begin{Lem}[Answer Probability]
  Let $p(t_1,t_2,\ldots,t_n)$ be an answer to a query in Px. Let $\phi$
  be the explanation graph such that $p(t_1,t_2,\ldots,t_n,\phi)$  is the
  answer to the corresponding query over the transformed program.    
  Then the probability of $p(t_1,t_2,\ldots,t_n) =  \id{prob}(\phi)$. 
  \hfill $\Box$
\end{Lem}
}

\subsection{Syntax and Semantics of Parameterized Px Programs}
Parameterized Px, called Param-Px for short, is a further extension of the
Px language.  The first feature of this extension is the specification
of \emph{populations} and \emph{instances} to specify ranges of
instance parameters of \texttt{msw}s.  

\begin{Def}[Population]\label{def:population}
  A \emph{population} is a named finite set, with a specified
  cardinality.
  A population has the following properties:
  \begin{enumerate}
  \item Elements of a population may be atomic, or depth-bounded
    ground terms.  
  \item Elements of a population are totally ordered using
    the default term order.  
  \item Distinct populations are disjoint. \hfill $\Box$
  \end{enumerate}
\end{Def}
Populations and their cardinalities are specified in a Param-Px
program by \texttt{population} facts.  
For example, the program in Figure~\ref{fig:intro-ex}(a) defines a
population named \texttt{coins} of size $100$.   The individual
elements of this set are left unspecified.   When necessary,
\texttt{element/2} facts may be used to define distinguished elements of a
population.  For example, \texttt{element(fred, persons)} defines a
distinguished element ``\texttt{fred}'' in population persons.  
In presence of \texttt{element} facts, elements of a population are
ordered as follows.  The order of \texttt{element} facts specifies the order among
the distinguished elements, and all distinguished elements 
occur before other unspecified elements in the order.

\begin{Def}[Instance]\label{def:instance}
  An \emph{instance} is an element of a population. In a Param-Px program, a
  built-in predicate \texttt{in/2} can be used to draw an instance from a
  population.  All instances of a population can be drawn by backtracking over
  \texttt{in}.  \hfill $\Box$
\end{Def}
For example, in Figure~\ref{fig:intro-ex}(a), \texttt{X in coins} binds
\texttt{X} to an instance of population \texttt{coins}.  An \emph{instance
  variable} is one that occurs as the instance argument in an \texttt{msw}
predicate in a clause of a Param-Px program.  For example, in
Figure~\ref{fig:intro-ex}(a), \texttt{X} and \texttt{Y} in the clause defining
\texttt{twoheads} are instance variables.

\myparagraph{Constraints.} The second extension in Param-Px are atomic
constraints, of the form $\{t_1 = \ t_2\}$, $\{t_1 \not= t_2\}$ and $\{t_1
< t_2\}$, where $t_1$ and $t_2$ are variables or constants, 
 to compare instances of a population.   We use braces ``$\{\cdot\}$''
 to distinguish the constraints from Prolog built-in  comparison operators.

\myparagraph{Types.}  We use populations in a Param-Px program to
confer types to program variables.
Each variable that occurs in an ``\texttt{in}''
predicate is assigned a unique type.  More specifically, $X$ has type
$p$ if \texttt{$X$ in $p$} occurs in a program, where $p$ is a
population; and $X$ is untyped otherwise.   We extend this notion of
types to constants and switches as well.  A constant $c$ has type $p$ if
there is a fact \texttt{element($c$, $p$)}; and $c$ is untyped
otherwise.   A switch $s$ has type $p$ if there is an
\texttt{msw($s$, $X$, $t$)} in the program and $X$ has type $p$; and
$s$ is untyped otherwise.   % By $\xi:p$, where $\xi$ may be a
% variable, a constant, or a switch,  we denote that $\xi$
% is of type $p$;  if $\xi$ is untyped, we write $\xi:\top$.  

\begin{Def}[Well-typedness and Typability]\label{def:well-typed}
  A Param-Px program is \emph{well-typed} if:
  \begin{enumerate}
  \item For every constraint in the program of the form 
    $\{t_1 = t_2\}$, $\{t_1 \not= t_2\}$ or $\{t_1 < t_2\}$, the
    types of $t_1$ and $t_2$ are identical. 
  \item Types of arguments of every atom on the r.h.s. of a clause are
    identical to the types of corresponding parameters of l.h.s. atoms of matching
    clauses.
  \item Every switch in the program has a unique type.
  \end{enumerate}
  A Param-Px program is \emph{typable} if we can add literals
  of the form \texttt{$X$ in $p$} (where $p$ is a population) to r.h.s. of
  clauses such that the resulting program is well-typed. \hfill $\Box$
\end{Def}
The first two conditions of well-typedness ensure that only instances from the
same population are compared in the program.  The last condition imposes that
instances of random variables generated by switch $s$ are all indexed by
elements drawn from the same population. In the rest of the paper, unless
otherwise specified, we assume all Param-Px programs under consideration are
well-typed.

\paragraph{Semantics of Param-Px Programs.}  Each Param-Px program can
be readily transformed into a non-parameterized ``ordinary'' Px
program.  Each \texttt{population} fact is used to generate a set of
\texttt{in/2} facts enumerating the elements of the population.  Other
constraints are replaced by their counterparts is Prolog: e.g. $\{X <
Y\}$ with $X \mathtt{<} Y$.  Finally, each \texttt{msw($s$,$i$,$t$)}
is preceded by $i\ \mathtt{in}\ p$ where $p$ is the type of $s$.  The
semantics of the original parameterized program is defined by the
semantics of the transformed program.

%%% Local Variables: 
%%% mode: latex
%%% TeX-master: "main"
%%% End: 

%\input{generation}
\section{Lifted Explanations}\label{sec:liftedexp}
In this section we formally define \emph{lifted explanation graphs}.
These are a generalization of {ground explanation graphs} defined
earlier, and are introduced in order to represent ground explanations
compactly.  As illustrated in Figure~\ref{fig:intro-ex} in
Introduction, the compactness of lifted explanations is a result of
summarizing the instance information.  Constraints over instances form
a basic building block of lifted explanations.  We use the following
constraint domain for this purpose.

\subsection{Constraints on Instances}

\begin{Def}[Instance Constraints]\label{def:interval-constraints}
  Let ${\cal V}$ be a set of instance variables, with subranges of integers
  as domains, such that $m$ is the largest positive integer in the
  domain of any variable.  Atomic constraints on instance variables are of one of the
  following two forms: $X < aY \pm k$, $X = aY \pm k$, where $X,Y \in
  {\cal V}$, $a \in  {0,1}$, where $k$ is a non-negative integer $\leq m+1$.   
  The language of constraints over bounded integer intervals, denoted
  by ${\cal L}({\cal V}, m)$, is a set of formulae $\eta$, where $\eta$ is a
  non-empty set of atomic constraints representing their conjunction.
\end{Def}
Note that each formula in ${\cal L}({\cal V}, m)$ is a convex
region in ${\mathbb{Z}}^{\lvert V \rvert}$, and hence is closed under conjunction
and existential quantification.  

Let $\id{vars}(\eta)$ be the set of instance variables in an instance
constraint $\eta$.   A substitution
$\sigma:\id{vars}(\eta) \rightarrow [1..m]$ that maps each variable to an
element in its domain is a \emph{solution} to $\eta$ if each constraint in
$\eta$ is satisfied by the mapping. The set of all solutions of $\eta$ is
denoted by $\sem{\eta}$. The constraint formula $\eta$ is unsatisfiable if
$\sem{\eta} = \emptyset$.  We say that $\eta \models \eta'$ if every
$\sigma \in \sem{\eta}$ is a solution to $\eta'$. 

Note also that instance constraints are a subclass of the well-known
integer octagonal constraints~\cite{Mine2006} and can be represented
canonically by difference bound matrices (DBMs)~\cite{Yovine98,LLPW97}, permitting efficient algorithms for
conjunction and existential quantification.  Given a constraint on
$n$ variables, a DBM is a $(n+1) \times (n+1)$ matrix with rows and columns
indexed by variables (and a special ``zero'' row and column).  For
variables $X$ and $Y$, the entry in cell $(X,Y)$ of a DBM represents
the upper bound on $X-Y$.   For variable $X$, the value at cell $(X,0)$
is $X$'s upper bound and the value at cell $(0,X)$ is the negation of
$X$'s lower bound.  

Geometrically, each entry
in the DBM representing a $\eta$ is a ``face''
of the region representing $\sem{\eta}$.   Negation of an instance
constraint $\eta$ can be represented by a set of mutually exclusive instance
constraints.  Geometrically, this can be seen as the set of convex regions
obtained by complementing the ``faces'' of the region representing
$\sem{\eta}$.   Note that when $\eta$ has $n$ variables, the number of instance constraints in $\neg
\eta$ is bounded by the number of faces of $\sem{\eta}$, and hence by $O(n^2)$.

Let $\neg \eta$ represent the set of mutually exclusive instance constraints representing
the negation of $\eta$.   Then the disjunction of two instance
constraints $\eta$ and $\eta'$ can be represented by the set of
mutually exclusive instance constraints $(\eta \land \neg \eta') \cup
(\eta' \land \neg \eta) \cup \{\eta \land \eta'\}$, where we overload
$\land$ to represent the element-wise conjunction of an
instance constraint with a set of constraints.  

An existentially quantified formula of the form $\exists X.\eta$ can
be represented by a DBM obtained by removing the rows and columns
corresponding to $X$ in the DBM representation of $\eta$.  We denote 
this simple procedure to obtain $\exists X.\eta$ from $\eta$ by $Q(X,
\eta)$.    

\begin{Def}[Range]\label{def:range}
  Given a constraint formula $\eta \in {\cal L}({\cal V},m)$, and
  $X \in \id{vars}(\eta)$, let
  $\sigma_X(\eta) = \{v \mid \sigma \in \sem{\eta}, \sigma(X) = v\}$. Then
  $\id{range}(X,\eta)$ is the interval $[l,u]$, where $l = min(\sigma_X(\eta))$
  and $u = max(\sigma_X(\eta))$.
\end{Def}

Since the constraint formulas represent convex regions, it follows
that each variable's range will be an interval.  Note that range of a
variable can be readily obtained in constant time from the entries for
that variable in the zero row and zero column of the constraint's DBM
representation.

\subsection{Lifted Explanation Graphs}
\begin{Def}[Lifted Explanation Graph]\label{def:lifted-expl-graph}
Let $S$ be the set of ground switches in a Param-Px program $P$, $D_s$ be
the domain of switch $s \in S$, $m$ be the sum of the cardinalities of
all populations in $P$ and $C$ be the set of distinguished elements
of the populations in $P$. A \emph{lifted explanation graph} over variables
${\cal V}$ is a pair $(\Omega:\eta, \psi)$ which satisfies the following conditions
\begin{enumerate}
\item $\Omega:\eta$ is the notation for $\exists \Omega.\eta$, where
  $\eta \in {\cal L}({\cal V},m)$ is either a satisfiable constraint formula,
  or the single atomic constraint $\texttt{false}$ and
  $\Omega \subseteq \id{vars}(\eta)$ is the set of quantified variables in
  $\eta$. When $\eta$ is $\texttt{false}$, $\Omega = \emptyset$.
  \item $\psi$ is a singly rooted DAG which satisfies the following conditions
    \begin{itemize}
      \item Internal nodes are labeled $(s,t)$ where $s \in S$ and $t \in {\cal
          V}\cup C$.
      \item Leaves are labeled either $0$ or $1$.
      \item Each internal node has an outgoing edge for each outcome $\in D_s$.
      \item If a node labeled $(s,t)$ has a child labeled $(s',t')$ then $\eta
        \models t < t'$ or $\eta \models t = t'$ and $(s,c) \prec (s',c)$ for
        any ground term $c$ (see Def. \ref{def:explanation-graph}).
    \end{itemize}
\end{enumerate}
\end{Def}

Similar to ground explanation graphs (Def. \ref{def:explanation-graph}), the DAG
components of the lifted explanation graphs are represented by textual patterns
$(s,t)[\alpha_i:\psi_i]$ where $(s,t)$ is the label of the root and $\psi_i$ is
the DAG associated with the edge labeled $\alpha_i$. Irrelevant parts may
denoted ``$\_$'' to reduce clutter. In the lifted explanation graph shown in
Figure~\ref{fig:intro-ex}(c), the $\Omega:\eta$ part would be $\{X,Y\}:X<Y$. We,
now define the standard notion of bound and free variables over lifted
explanation graphs.
\begin{Def}[Bound and free variables]\label{def:bound-free}
  Given a lifted explanation graph $(\Omega:\eta, \psi)$, a variable
  $X \in \id{vars}(\eta)$, is called a bound variable if $X \in \Omega$,
  otherwise its called a free variable.
\end{Def}
The lifted explanation graph is said to be \emph{well-structured} if every pair
of nodes $(s,X)$ and $(s',X)$ with the same bound variable $X$, have a common
ancestor with $X$ as the instance variable. In the rest of the paper, we assume
that the lifted explanation graphs are well-structured.

\begin{Def}[Substitution operation]\label{def:substitution}
  Given a lifted explanation graph $(\Omega:\eta,\psi)$, a variable
  $X \in \id{vars}(\eta)$, the substitution of $X$ in the lifted explanation
  graph with a value $k$ from its domain, denoted by $(\Omega:\eta,\psi)[k/X]$
  is defined as follows:
\begin{align*}
  (\Omega:\eta, \psi)[k/X] &= (\emptyset:\{\texttt{false}\}, 0), \mbox{ if }
                             \eta[k/X] \mbox{ is unsatisfiable}\\
  (\Omega:\eta, \psi)[k/X] &= (\Omega\setminus\{X\}:\eta[k/X], \psi[k/X]),
                             \mbox{ if } \eta[k/X] \mbox{ is satisfiable}\\
  ((s,t)[\alpha_i:\psi_i])[k/X] &= (s,k)[\alpha_i:\psi_i[k/X]], \mbox{ if } t =
                                  X\\
  ((s,t)[\alpha_i:\psi_i])[k/X] &= (s,t)[\alpha_i:\psi_i[k/X]], \mbox{ if } t
                                  \neq X\\
  0[k/X] &= 0\\
  1[k/X] &= 1
\end{align*}
\end{Def}

In the above definition, $\eta[k/X]$ refers to the standard notion of
substitution. The definition of substitution operation can be generalized to
mappings on sets of variables. Let $\sigma$ be a substitution that maps
variables to their values.  By $(\Omega:\eta,\psi)\sigma$ we denote the lifted
explanation graph obtained by sequentially performing substitution operation on
each variable $X$ in the domain of $\sigma$.

\begin{Lem}[Substitution lemma]\label{lem:substitution}
  If $(\Omega:\eta,\psi)$ is a lifted explanation graph, and
  $X \in \id{vars}(\eta)$, then $(\Omega:\eta,\psi)[k/X]$ where $k$ is a value in
  domain of $X$, is a lifted explanation graph.
\end{Lem}

When a substitution $[k/X]$ is applied to a lifted explanation graph, and
$\eta[k/X]$ is unsatisfiable, the result is $(\emptyset:\{\texttt{false}\}, 0)$
which is clearly a lifted explanation graph. When $\eta[k/X]$ is satisfiable,
the variable is removed from $\Omega$ and occurrences of $X$ in $\psi$ are
replaced by $k$. The resultant DAG clearly satisfies the conditions imposed by
the Def \ref{def:lifted-expl-graph}. Finally we note that a ground explanation
graph $\phi$ (Def. \ref{def:explanation-graph}) is a trivial lifted explanation
graph $(\emptyset:\{\texttt{true}\}, \phi)$. This constitutes the informal proof
of lemma \ref{lem:substitution}.

\subsection{Semantics of Lifted Explanation Graphs}
The meaning of a lifted explanation graph $(\Omega:\eta,\psi)$ is given by the
ground explanation tree represented by it.
\begin{Def}[Grounding]\label{def:grounding}
  Let $(\Omega:\eta,\psi)$ be a closed lifted explanation graph, i.e., it has no
  free variables. Then the ground explanation tree represented by
  $(\Omega:\eta,\psi)$, denoted $\id{Gr}((\Omega:\eta,\psi))$, is given by the
  function $Gr(\Omega,\eta,\psi)$.  When $\sem{\eta}=\emptyset$, then
  $Gr(\_,\eta,\_)=0$. We consider the cases when $\sem{\eta} \neq
  \emptyset$. The grounding of leaves is defined as $Gr(\_,\_,0)=0$ and
  $Gr(\_,\_,1)=1$. When the instance argument of the root is a constant,
  grounding is defined as
  $Gr(\Omega,\eta,(s,t)[\alpha_i:\psi_i]) =
  (s,t)[\alpha_i:Gr(\Omega,\eta,\psi_i)]$. When the instance argument is a bound
  variable, the grounding is defined as
  $\id{Gr}(\Omega,\eta,(s,t)[\alpha_i:\psi_i]) \equiv \bigvee_{c \in
    range(t,\eta)}(s,c)[\alpha_i:\id{Gr}(\Omega\setminus\{t\},\eta[c/t],\psi_i[c/t])]$.
\end{Def}

In the above definition $\psi[c/t]$ represents the tree obtained by replacing
every occurrence of $t$ in the tree with $c$. The disjunct
$(s,c)[\alpha_i:\id{Gr}(\Omega\setminus\{t\},\eta[c/t],\psi_i[c/t])]$ in the
above definition is denoted $\phi_{(s,c)}$ when the lifted explanation graph is
clear from the context. The grounding of the lifted explanation graph in Figure
\ref{fig:intro-ex}(c) is shown in Figure \ref{fig:grounding} when there are
three coins. Note that in the figure the disjuncts corresponding to the
grounding of the left subtree of the lifted explanation graph have been combined
using the $\lor$ operation. In a similar way the two disjuncts corresponding to
the root would also be combined. Further note that the third disjunct corresponding
to grounding $X$ with $3$ is ommitted because, it has all $0$ children and would 
therefore get collapsed into a single $0$ node.
\begin{figure}
  \centering
  \begin{adjustbox}{scale=0.75}
    \begin{tikzpicture}[shorten >=1pt,%
          inner sep = 0.5mm, %
          node distance=2.25cm,%
          every state/.style={circle}, 
          on grid,auto,%
          initial text=,%
          ellipsis/.style={dotted},%
          bend angle=30]

          \node[state,rectangle] (s11) {$(toss,1)$};
          \node[state,rectangle] (s21) [below left=of s11] {$(toss,2)$};
          \node[state] (s22) [below right=of s11] {$0$};
          \node[state] (s31) [below left=of s21] {$1$};
          \node[state,rectangle] (s32) [below right=of s21] {$(toss,3)$};
          \node[state] (s41) [below left=of s32] {$1$};
          \node[state] (s42) [below right=of s32] {$0$};

          \node (s23) [right=of s22] {$\bigvee$};

          \node[state,rectangle] (s24) [right=of s23] {$(toss,3)$};
          \node[state,rectangle] (s12) [above right=of s24] {$(toss,2)$};
          \node[state] (s25) [below right=of s12] {$0$};
          \node[state] (s33) [below left=of s24] {$1$};
          \node[state] (s34) [below right=of s24] {$0$};

          \path[->]
          (s11) edge [swap] node {h} (s21)
          (s11) edge node {t} (s22)
          (s21) edge [swap] node {h} (s31)
          (s21) edge node {t} (s32)
          (s32) edge [swap] node {h} (s41)
          (s32) edge node {t} (s42)
          (s12) edge [swap] node {h} (s24)
          (s12) edge node {t} (s25)
          (s24) edge [swap] node {h} (s33)
          (s24) edge node {t} (s34);
    \end{tikzpicture}
  \end{adjustbox}
\caption{Lifted expl. graph grounding example}
\label{fig:grounding}
\end{figure}
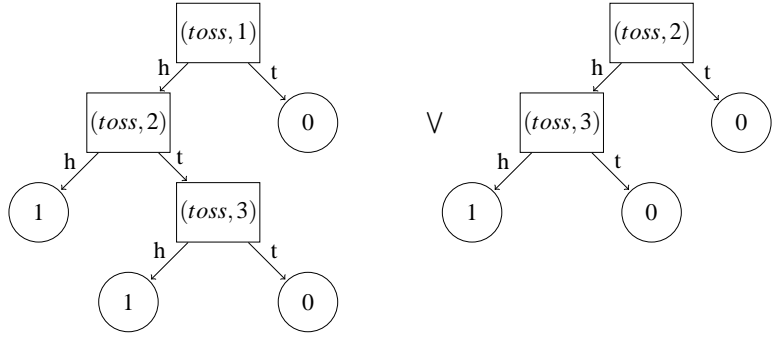

\subsection{Operations on Lifted Explanation Graphs}
\myparagraph{And/Or Operations.}  
Let $(\Omega:\eta,\psi)$ and $(\Omega':\eta',\psi')$ be two lifted explanation
graphs. We now define ``$\land$" and ``$\lor$'' operations on them. The
``$\land$" and ``$\lor$'' operations are carried out in two steps. First, the
constraint formulas of the inputs are combined.  The key issue in defining these
operations is to ensure the right order among the graph nodes (see criterion~3
of Def.~\ref{def:lifted-expl-graph}).  However, the free variables in the
operands may have \emph{no known order} among them. Since, an arbitrary order
cannot be imposed, the operations are defined in a \emph{relational}, rather
than functional form. We use the notation
$(\Omega:\eta,\psi) \oplus (\Omega':\eta',\psi') \rightarrow
(\Omega'':\eta'',\psi'')$ to denote that $(\Omega'':\eta'',\psi'')$ is \emph{a}
result of $(\Omega:\eta,\psi) \oplus (\Omega':\eta',\psi')$. When an operation
returns multiple answers due to ambiguity on the order of free variables, the
answers that are inconsistent with the final order are discarded.  We assume
that the variables in the two lifted explanation graphs are standardized apart
such that the bound variables of $(\Omega:\eta,\psi)$ and
$(\Omega':\eta', \psi')$ are all distinct, and different from free variables of
$(\Omega:\eta, \psi)$ and $(\Omega':\eta',\psi')$. Let
$\psi=(s,t)[\alpha_i:\psi_i]$ and $\psi'=(s',t')[\alpha'_i:\psi'_i]$.

\myparagraph{Combining constraint formulae}
\begin{description}
\item [$Q(\Omega,\eta) \land Q(\Omega',\eta')$ is unsatisfiable.] Then the orders
among free variables in $\eta$ and $\eta'$ are incompatible.
\begin{itemize}
\item The $\land$ operation is defined as
  $ (\Omega:\eta,\psi) \land (\Omega':\eta',\psi') \rightarrow
  (\emptyset:\{\texttt{false}\},0) $
\item The $\lor$ operation simply returns the two inputs as outputs:
\begin{align*} 
(\Omega:\eta,\psi) \lor (\Omega':\eta',\psi') \rightarrow & (\Omega:\eta,\psi)\\
(\Omega:\eta,\psi) \lor (\Omega':\eta',\psi') \rightarrow & (\Omega':\eta',\psi')
\end{align*}
\end{itemize}

\item[$Q(\Omega,\eta) \land Q(\Omega',\eta')$ is satisfiable.] The orders among free
variables in $\eta$ and $\eta'$ are compatible
\begin{itemize}
\item The $\land$ operation is defined as follows
$
(\Omega:\eta,\psi) \land (\Omega':\eta',\psi') \rightarrow
(\Omega\cup\Omega':\eta \land \eta', \psi \land \psi')
$
\item The $\lor$ operation is defined as
\begin{align*}
(\Omega:\eta,\psi) \lor (\Omega':\eta',\psi') \rightarrow & (\Omega \cup \Omega': \eta \land \neg \eta', \psi)\\
(\Omega:\eta,\psi) \lor (\Omega':\eta',\psi') \rightarrow & (\Omega \cup \Omega': \eta' \land \neg \eta, \psi')\\
(\Omega:\eta,\psi) \lor (\Omega':\eta',\psi') \rightarrow & (\Omega \cup \Omega': \eta \land \eta', \psi \lor \psi')
\end{align*}
\end{itemize}
\end{description}

\myparagraph{Combining DAGs}
Now we describe $\land$ and $\lor$ operations on the two DAGs $\psi$ and $\psi'$
in the presence of a single constraint formula. The general form of the
operation is $(\Omega:\eta, \psi \oplus \psi')$.
\begin{description}
\item [Base cases: ] The base cases are as follows (symmetric base cases are defined
  analogously). 
\begin{align*}
  (\Omega:\eta, 0 \lor \psi') \rightarrow & (\Omega:\eta, \psi')\\
  (\Omega:\eta, 1 \lor \psi') \rightarrow & (\Omega:\eta, 1)\\
  (\Omega:\eta, 0 \land \psi') \rightarrow & (\Omega:\eta, 0)\\
  (\Omega:\eta, 1 \land \psi') \rightarrow & (\Omega:\eta, \psi')
\end{align*}
\item [Recursion: ] When the base cases do not apply, we try to compare the
  roots of $\psi$ and $\psi'$. The root nodes are compared as follows: We say
  $(s,t) = (s',t')$ if $\eta \models t = t'$ and $s=s'$, else $(s,t) < (s',t')$
  (analogously $(s',t') < (s,t)$) if $\eta \models t < t'$ or
  $\eta \models t=t'$ and $(s,c) \prec (s',c)$ for any ground term $c$. If
  neither of these two relations hold, then the roots are not comparable and its
  denoted as $(s,t) \not \sim (s',t')$.
\begin{description}
  \item [a. ]$(s,t) < (s',t')$
   \[
     (\Omega:\eta, \psi \oplus \psi') \rightarrow (\Omega:\eta, (s,t)[\alpha_i:
     \psi_i \oplus \psi'])
   \]
 \item [b. ]$(s',t') < (s,t)$
   \[
     (\Omega:\eta, \psi \oplus \psi') \rightarrow (\Omega:\eta, (s',t')[\alpha'_i:
     \psi \oplus \psi'_i])
   \]
 \item [c. ]$(s,t) = (s',t')$
   \[
     (\Omega:\eta, \psi \oplus \psi') \rightarrow (\Omega:\eta, (s,t)[\alpha_i:
     \psi_i \oplus \psi'_i])
   \]
 \item [d. ]$(s,t) \not \sim (s',t')$
   \begin{description}
   \item [i. ] $t$ is a free variable or a constant, and $t'$ is a free
     variable (the symmetric case is analogous).
       \begin{align*}
         (\Omega:\eta, \psi \oplus \psi') \rightarrow & (\Omega:\eta\land t<t',
                                                        \psi \oplus \psi')\\
         (\Omega:\eta, \psi \oplus \psi') \rightarrow & (\Omega:\eta\land t=t',
                                                        \psi \oplus \psi')\\
         (\Omega:\eta, \psi \oplus \psi') \rightarrow & (\Omega:\eta\land t'<t,
                                                        \psi \oplus \psi')
       \end{align*}
     \item [ii. ] $t$ is a free variable or a constant and $t'$ is a bound
       variable (the symmetric case is analogous)
       \begin{align*}
         (\Omega:\eta, \psi \oplus \psi') \rightarrow & & (\Omega:\eta\land t < t',
         \psi \oplus \psi')\\
         & \vee & (\Omega:\eta\land t=t',\psi\oplus\psi')\\
         & \vee & (\Omega:\eta\land t' < t, \psi \oplus \psi')
       \end{align*}
       Note that in the above definition, all three lifted explanation graphs
       use the same variable names for bound variable $t'$. Lifted explanation
       graphs can be easily standardized apart on the fly, and henceforth we
       assume that the operation is applied as and when required.
     \item [iii. ] $t$ and $t'$ are bound variables. Let
       $range(t,\eta) = [l_1,u_1]$ and $range(t',\eta) = [l_2,u_2]$.  We can
       conclude that $range(t,\eta)$ and $range(t',\eta)$ are overlapping,
       otherwise $(s,t)$ and $(s',t')$ could have been ordered. Without loss of
       generality, we assume that $l_1 \leq l_2$ and we consider various cases
       of overlap as follows:

       When $l_1=l_2$ and $u_1=u_2$
       \begin{align*}
         (\Omega:\eta, \psi \oplus \psi') & \rightarrow (\Omega\cup\{t''\}:\eta \land l_1-1<t'' \land t''-1 < u_1 \land t'' < t \land t'' < t',
         \\
         & (s,t'')[ \alpha_i: \\
         & (\psi_i[t''/t] \oplus \psi'_i[t''/t']) \lor \\
         & (\psi_i[t''/t] \oplus \psi') \lor \\
         & (\psi'_i[t''/t'] \oplus \psi)])
       \end{align*}
       % \begin{align*}
       %   (\Omega:\eta, \psi \oplus \psi') & \rightarrow (\Omega\cup\{t''\}:\eta \land \eta[t''/t,t''/t'] \land t'' < t \land t'' < t',
       %   \\
       %   & (s,t'')[ \alpha_i: \\
       %   & (\psi_i[t''/t] \oplus \psi'_i[t''/t']) \lor \\
       %   & (\psi_i[t''/t] \oplus \psi') \lor \\
       %   & (\psi'_i[t''/t'] \oplus \psi)])
       % \end{align*}
       When $\oplus$ is $\lor$ the result can be simplified as
       \[
         (\Omega:\eta, \psi \oplus \psi') \rightarrow (\Omega\cup\{t''\}\setminus\{t,t'\}:\eta[t''/t,t''/t'], (s,t'')[\alpha_i: \psi_i[t''/t] \lor \psi'_i[t''/t']])
       \]
       When $l_1=l_2$ and $u_1 < u_2$ the result is
       \[
         (\Omega:\eta \land t'-1
                                                        < u_1,
                                                        \psi \oplus \psi')
                                                      \vee (\Omega:\eta \land u_1 <
                                                        t', \psi \oplus
                                                        \psi')
       \]
       When $l_1=l_2$ and $u_2 < u_1$ the result is
       \[
       (\Omega:\eta \land t=t',
                                                        \psi \oplus \psi')
                                                      \vee (\Omega:\eta \land u_2 <
                                                        t, \psi \oplus
                                                        \psi')
       \]
       When $l_1 < l_2$ and $u_1 = u_2$ the result is
       \[
         (\Omega:\eta \land t=t',
                                                        \psi \oplus \psi')
                                                      \vee (\Omega:\eta \land t <
                                                        l_2, \psi \oplus
                                                        \psi')
       \]
       When $l_1 < l_2$ and $u_1 < u_2$ the result is
       \[
         (\Omega:\eta \land u_1 < t',
                                                          \psi \oplus \psi')
                                                      \vee (\Omega:\eta \land t <
                                                        l_2 \land t'-1 < u_1, \psi \oplus
                                                        \psi')
                                                      \vee (\Omega:\eta \land
                                                        t=t', \psi \oplus \psi')
       \]
       When $l_1 < l_2$ and $u_2 < u_1$ the result is
       \[
         (\Omega:\eta \land u_2 < t,
                                                          \psi \oplus \psi')
                                                      \vee (\Omega:\eta \land t <
                                                        l_2, \psi \oplus
                                                        \psi')
                                                      \vee (\Omega:\eta \land
                                                        t=t', \psi \oplus \psi')
       \]
   \end{description}
\end{description}
\end{description}

\begin{restatable}[Correctness of ``$\land$'' and ``$\lor$'' operations]{Lem}{andorcorrect}
\label{lem:andor-correctness}
  Let $(\Omega:\eta,\psi)$ and $(\Omega':\eta',\psi')$ be two lifted explanation
  graphs with free variables $\{X_1, X_2\ldots,X_n\}$.  Let $\Sigma$ be the set
  of all substitutions mapping each $X_i$ to a value in its domain. Then, for every 
  $\sigma \in \Sigma$, and $\oplus \in \{\land, \lor\}$
\[
\id{Gr}(((\Omega:\eta,\psi) \oplus (\Omega':\eta',\psi'))\sigma) =
\id{Gr}((\Omega:\eta,\psi)\sigma) \oplus \id{Gr}((\Omega':\eta',\psi')\sigma)
\]
\end{restatable}

\myparagraph{Quantification.} 

\begin{Def}[Quantification]
Operation $\id{quantify}((\Omega:\eta,\psi),X)$
changes a free variable $X \in \id{vars}(\eta)$ to a quantified variable. It is 
defined as
\[
\id{quantify}((\Omega:\eta,\psi),X) = (\Omega\cup\{X\}:\eta,\psi),\mbox{ if }X
\in \id{vars}(\eta)
\]
\end{Def}

\begin{restatable}[Correctness of $\id{quantify}$]{Lem}{quantifycorrect}
\label{lem:quantify-correctness}
  Let $(\Omega:\eta,\psi)$ be a lifted explanation graph, let $\sigma_{-X}$ be a
  substitution mapping all the free variables in $(\Omega:\eta,\psi)$ except $X$
  to values in their domains. Let $\Sigma$ be the set of mappings $\sigma$ such
  that $\sigma$ maps all free variables to values in their domains and is
  identical to $\sigma_{-X}$ at all variables except $X$. Then the following holds
\[
  \id{Gr}(\id{quantify}((\Omega:\eta,\psi), X)\sigma_{-X}) = \bigvee_{\sigma \in \Sigma}
  \id{Gr}((\Omega:\eta,\psi)\sigma)
\]
\end{restatable}

\myparagraph{Construction of Lifted Explanation Graphs}
Lifted explanation graphs for a query are constructed by transforming the 
Param-Px program ${\cal P}$ into one that explicitly 
constructs a lifted explanation graph, following a similar procedure to
the one outlined in Section~\ref{sec:parampx} for constructing ground
explanation graphs.  The main difference is the use of existential
quantification.  Let $A \when G$ be a program clause, and
$\id{vars}(G)  - \id{vars}(A)$ be the set of variables in $G$ and not
in $A$.  If any of these variables has a type, then it means that the
variable used as an instance argument in $G$ is existentially
quantified.  Such clauses are then translated as $\id{head}(A, E_h)
\when \id{exp}(G, E_g), \id{quantify}(E_g, V_s, E_h)$, where $V_s$ is
the set of typed variables in $\id{vars}(G)  - \id{vars}(A)$.   
A minor difference is the treatment of constraints:  \id{exp} is
extended to atomic constraints $\varphi$ such that $\id{exp}(\varphi, E)$
binds $E$ to $(\emptyset:\{\varphi\}, 1)$.

We order the populations and map the elements of the populations to natural
numbers as follows. The population that comes first in the order is mapped to
natural numbers in the rangle $1..m$, where $m$ is the cardinality of this
population. Any constants in this population are mapped to natural numbers in
the low end of the range. The next population in the order is mapped to natural
numbers starting from $m+1$ and so on. Thus, each typed variable is assigned a
domain of contiguous positive values.

The rest of the program transformation remains the same, the underlying
graphs are  constructed using the lifted operators. In order to illustrate some
of the operations described in this section, we present another example Param-Px
program (Figure \ref{ex:diceprog}) and show the construction of lifted explanation
graphs (Figure \ref{fig:dicegraph1} and Figure \ref{fig:dicegraph2}).

\myparagraph{Dice Example.} The listing in Figure \ref{ex:diceprog} shows a
Param-Px program, where the query tests if on rolling a set of dice, we got
atleast two ``ones'' or two ``twos''. The lifted explanation graph for the first
clause is obtained by first taking a conjunction of three lifted explanation
graphs: one for the constraint $\{X < Y\}$ and two corresponding to the
$\texttt{msw}$ goals and then quantifying the variables $X$ and $Y$. The lifted
explanation graph for the second clause is constructed in a similar
fashion. They are shown in Figure \ref{fig:dicegraph1}. Note, that we ommitted
the edge labels to avoid clutter as they are obvious. Next the two lifted
explanation graphs need to be combined by an $\lor$ operation. Let us denote the
two lifted explanation graphs to be combined as $(\Omega:\eta,\psi)$ and
$(\Omega':\eta',\psi')$. When $\lor$ operation is sought to be performed, it is
noticed that $Q(\Omega,\eta) \land Q(\Omega',\eta')$ is satisfiable. In fact,
$Q(\Omega,\eta)$ and $Q(\Omega,\eta')$ both evaluate to
$\texttt{true}$. Therefore, only one result is returned
\[
  (\Omega\cup\Omega':\eta \land \eta', \psi \lor \psi')
\]
The operation to be performed is the one described in recursive case $d$,
subcase $iii$ in the description earlier. Here we note that
$(roll,X) \not \sim (roll,X')$ and further the range of both instance variables
is same. Therefore the following simplified operation is performed
\[
  (\Omega:\eta, \psi \oplus \psi') \rightarrow (\Omega\cup\{X''\}\setminus\{X,X'\}:\eta[X''/X,X''/X'], (s,X'')[\alpha_i: \psi_i[X''/X] \lor \psi'_i[X''/X']])
\]
The result is shown in Figure \ref{fig:dicegraph2}.
\begin{figure}
\centering
\begin{lstlisting}[name=DiceExample]
% Get two "ones" or two "twos"
q :-
    X in dice,
    msw(roll, X, 1),
    Y in dice,
    {X < Y},
    msw(roll, Y, 1).

q :-
    X in dice,
    msw(roll, X, 2),
    Y in dice,
    {X < Y},
    msw(roll, Y, 2).

% Cardinality of dice:
:- population(dice, 100).

% Distribution parameters:
:- set_sw(roll, categorical([1:1/6, 2:1/6, 3:1/6, 4:1/6, 5:1/6, 6:1/6])).
\end{lstlisting}
\caption{Rolling dice Param-Px example}
\label{ex:diceprog}
\end{figure}

\begin{figure}
\centering
\begin{adjustbox}{scale=0.75}
\begin{tikzpicture}[shorten >=1pt,%
          inner sep = 0.5mm, %
          node distance=2cm,%
          every state/.style={circle}, 
          on grid,auto,%
          initial text=,%
          ellipsis/.style={dotted},%
          bend angle=30,
          level 1/.style={sibling distance=1.5cm}]

          \node[state,rectangle] (c11) at (0,1.25) {$\exists X \exists Y. X < Y$};
          \node[state,rectangle] (s11) at (0,0) {$(\mathtt{roll},X)$}
            child {node[state,rectangle] {$(\mathtt{roll},Y)$}
              child {node[state] {$1$}}
              child {node[state] {$0$}}
              child {node[state] {$0$}}
              child {node[state] {$0$}}
              child {node[state] {$0$}}
              child {node[state] {$0$}}}
            child {node[state] {$0$}}
            child {node[state] {$0$}}
            child {node[state] {$0$}}
            child {node[state] {$0$}}
            child {node[state] {$0$}};

          \node[state,rectangle] (c12) at (10,1.25) {$\exists X' \exists Y'. X' < Y'$};
          \node[state,rectangle] (s12) at (10,0) {$(\mathtt{roll},X')$}
            child {node[state] {$0$}}
            child {node[state,rectangle] {$(\mathtt{roll},Y')$}
              child {node[state] {$0$}}
              child {node[state] {$1$}}
              child {node[state] {$0$}}
              child {node[state] {$0$}}
              child {node[state] {$0$}}
              child {node[state] {$0$}}}
            child {node[state] {$0$}}
            child {node[state] {$0$}}
            child {node[state] {$0$}}
            child {node[state] {$0$}};

\end{tikzpicture}
\end{adjustbox}
\caption{Lifted expl. graphs for clauses in dice example}
\label{fig:dicegraph1}
\end{figure}
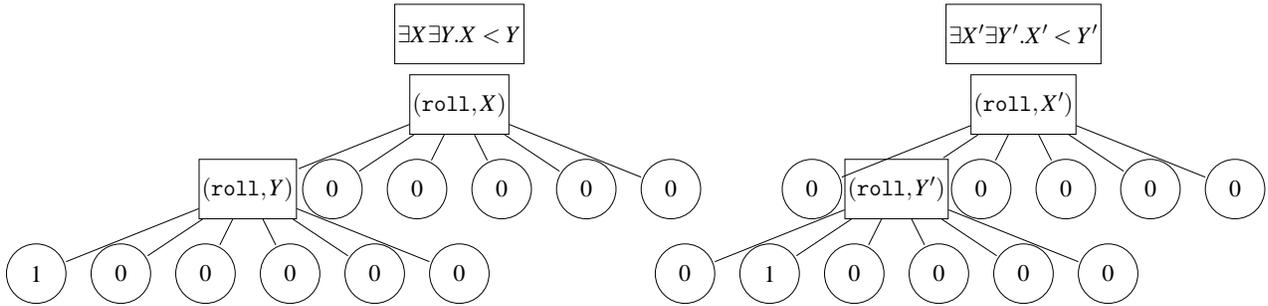

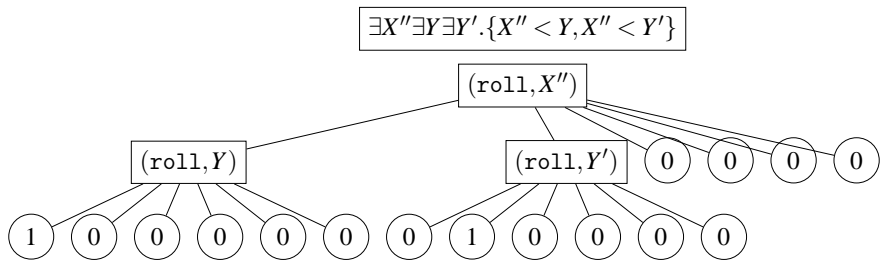
\begin{figure}
\centering
\begin{adjustbox}{scale=0.8}
\begin{forest}
  rounded/.style={circle, draw},
  squared/.style={rectangle,draw}
          [{$(\mathtt{roll},X'')$},squared,name=r
            [{$(\mathtt{roll},Y)$},squared
              [{$1$},rounded]
              [{$0$},rounded]
              [{$0$},rounded]
              [{$0$},rounded]
              [{$0$},rounded]
              [{$0$},rounded]]
            [{$(\mathtt{roll},Y')$},squared
              [{$0$},rounded]
              [{$1$},rounded]
              [{$0$},rounded]
              [{$0$},rounded]
              [{$0$},rounded]
              [{$0$},rounded]]
            [{$0$},rounded]
            [{$0$},rounded]
            [{$0$},rounded]
            [{$0$},rounded]]
          \node[draw] (c) at (r |- 0,30pt) {$\exists X'' \exists Y \exists Y'. \{X'' < Y, X'' < Y'\}$}[];
\end{forest}
\end{adjustbox}
\caption{Final lifted expl. graph for dice example}
\label{fig:dicegraph2}
\end{figure}

\comment{
Note that instance variables in a program range over populations.  We 
map instances to natural numbers when building the graph by treating
each population as a
range of natural numbers $1..n$, where $n$ is the cardinality of the
population.  Any constants in a population are mapped to natural numbers at
the low end of the range.  }

%%% ATTIC:
\comment{
Since multiple instances of a given switch are i.i.d
random variables, such summarization is natural and can help avoid redundant
computation. An example program is shown in Fig \ref{fig:lexpintropx}. Here the
query asks whether two coins landed heads from a population of coins. The ground
explanation tree in this scenario is shown in Fig
\ref{fig:lexpintroground}. Note that the root node and its left subtree are
similar to the right child of the root and its left subtree (with the instance
numbers being incremented by one). This kind of recursive structure arises due
t\section{Lifted Inference using Lifted
  Explanations}\label{sec:inference}

We now describe our two-step technique to compute answer probabilities
from a lifted explanation graph $\psi$.  In the first step, we derive a set of recurrences
from $\psi$, and in the second step, solve the
recurrences to obtain the probability.

Given a lifted explanation graph $\psi$ with $k$ free
variables,
we associate a function $f_\psi: N^k \rightarrow [0,1]$ that
maps $k$-tuples of natural numbers to reals in closed interval
$[0,1]$.  If the instance at the root of $\psi$ is quantified, we
associate two additional functions, $g_\psi: N^{k+1} \rightarrow
[0,1]$ and $h_\psi: N^{k+1} \rightarrow
[0,1]$ with $\psi$.  Since $\psi$ is well formed, we fix a 
total order on all variables in $\psi$, and use this order
consistently. We use $\vec{Z}_\psi$ to denote the sequence of
free variables of $\psi$ arranged in this fixed order.  

\comment{
\begin{Def}[Solutions to constraints]
Given a constraint $\eta$ and a sequence of $k$ 
variables $\vec{Z} = \langle Z_1, Z_2, \ldots Z_k$, 
the solution set of $\eta$, denoted by $\sem{\eta}$, is the largest subset of $N^k$
such that if $(j_1, j_2, \ldots, j_k) \in \sem{\eta}$ then
$\eta[j_1/Z_1, j_2/Z_2, \ldots,j_k/Z_k]$ is satisfiable.
\end{Def}

Given a constraint $\eta$ and a distinguished variable $X$, the
greatest lower bound on $X$ in any solution, denoted by
$\min_{X}(\eta)$, is an expression over other variables in $\eta$.
Similarly, the least upper bound on $X$ is denoted by
$\max_{X}(\eta)$.    For example, if $\eta = \{\langle Y<X \rangle,
\langle X<Z \rangle, \langle X\in p\rangle, \langle Y \in p \rangle,
\langle Z \in p \langle \}$ where $p$ is a population with cardinality
$n$, then $\min_X(\eta) = Y +1$ and $\max_X(\eta) = Z-1$.  Observe
also that $\max_Y(\eta) = X-1$ and $\max_Z(\eta) = n$. 
}

\begin{Def}[Probability Recurrences] 
Given a lifted explanation graph $\psi$, we define
$f_\psi$ (as well as  $g_\psi$ and $h_\psi$ wherever
applicable) based on the structure of $\psi$, as
follows:

\begin{description}
\item[Case 1:] $\psi = (\omega:\eta, 0)$:
\[
f_\psi(\vec{Z}_\psi) = 0
\]
\item[Case 2:] $\psi = (\omega:\eta, 1)$:
\[
f_\psi(\vec{Z}_\psi) = 
\left\{ \begin{array}{ll}
  1 & \mbox{if\ }\vec{Z}_\psi \in \sem{\eta}\\
  0 & \mbox{otherwise}
\end{array}\right.
\]

\item[Case 3:] $\psi = (\omega:\eta, (s,t)[\alpha_i:\psi_i])$ where $t$
  is a constant or free, i.e., $t \not \in \omega$:
\[
f_\psi(\vec{Z}_\psi) = 
\left\{ \begin{array}{ll}
\sum_i \pi_s(\alpha_i) \times
   f_{\psi_i}(\vec{Z}_{\psi_i})  & \mbox{if\ }\vec{Z}_\psi \in \sem{\eta}\\
  0 & \mbox{otherwise}
 \end{array}\right.
\]

\item[Case 4:] $\psi = (\omega:\eta, (s,X)[\alpha_i:\psi_i])$ and $X$
  is quantified at root of $\psi$, i.e. $X
  \in \omega$:

Let
  $ l = \min_{\eta}(X)$.  Then,
\[
\begin{array}{rl}
g_\psi(X,\vec{Z}_\psi) & = \sum_i \pi(\alpha_i)  \times f_{\psi_i}(\vec{Z}_{\psi_i})\\
h_\psi(X,\vec{Z}_\psi) &=
\left \{ \begin{array}{ll}
g(X,\vec{Z}_\psi) + [(1-g(X,\vec{Z}_\psi)) \times h_\psi(X+1,\vec{Z}_\psi)] & \mbox{if\ } X \leq \max_\eta(X)\\
f_{\psi_\bot}(\vec{Z}_{\psi_\bot}) & \mbox{otherwise}
\end{array} \right.\\
f_\psi(\vec{Z}_\psi) &= 
\left\{ \begin{array}{ll}
   h_\psi(l, \vec{Z}_\psi) & \mbox{if\ }\vec{Z}_\psi \in \sem{\eta}\\
  0 & \mbox{otherwise}
 \end{array}\right.
\end{array}
\]

\end{description}
\end{Def}

The recurrences defining $f_\psi$, $g_\psi$ and $h_\psi$ are
well-defined, and are computable, given values for the free variables
of $\psi$.  

\begin{Def}[Probability of Lifted Explanation
  Graph]\label{def:prob-lifted-expl}
Let $\psi$ be a closed lifted explanation graph.  Then, the
probability of explanations represented by the graph,  $\id{prob}(\psi)$,
is the value of $f_\psi()$. 
\end{Def}

\begin{Lem}[Correctness of Lifted Inference]
Let $\psi$ be a closed lifted explanation graph, and $\phi = Gr(\psi)$
be the corresponding ground explanation graph.  Then $\id{prob}(\psi)
= \id{prob}(\phi)$.  
\end{Lem}

Note that a graph with $k$ free
variables may generate a recurrence with $k+1$ parameters.  Each
recurrence equation itself is either of constant size or bounded by
the number of children of a node.  Using dynamic
programming (possibly implemented via tabling), a solution to the
recurrence equations can be computed in polynomial time.

\begin{Lem}[Efficiency of Lifted Inference]
Let $\psi$ be a closed lifted inference graph, $n$ be the size of the
largest population, and $k$ be the largest number of free variables at
the root of any subgraph of $\psi$.  Then, $f_\psi()$ can be computed in $O(|\psi|
\times n^{k+1})$ time.
\end{Lem}

There are two sources of further optimization in the generation and
evaluation of recurrences.  When constructing explanation graphs, we
liberally propagated constraints, anticipating possible future need.
Before generating the recurrences, we can prune this set to only those
that are \emph{used} in each subgraph.  This reduces the arity of
recurrence functions and promotes reuse of their results as well.
Moreover, certain recurrences may be transformed into closed form
formulae which can be more efficiently evaluated.  For instance, the
closed form formula for $h_2$ (Equation~\ref{eqn:h2},
page~\pageref{eqn:h2}) 
can be evaluated in $O(\log(n))$ time while a naive
evaluation of the recurrence takes $O(n)$ time.

\myparagraph{Other Examples.} 
There are a number of simple probabilistic models that cannot be
tackled by other lifted inference techniques but can be encoded
in Param-Px and solved using our technique.  For one such example,
consider an urn with $n$ balls, where the color of each ball is given
by a distribution.  Determining the probability that there are at
least two green balls is easy to phrase as a directed first-order
graphical model.  However, lifted inference over such models can no
longer be applied if we need to determine the probability of at least
two green or two red balls. The probability computation for one of these events
can be viewed as a generalization of noisy-OR probability computation, however
dealing with the union requires the handling of intersection of
the two events, due to which the $O(\log(N))$ time computation is no
longer feasible. 

For a more complex example, we use an instance of a \emph{collective
graphical model}~\cite{NIPS2011_4220}.  
In particular, consider a system of $n$ agents where each
agent moves between various states in a stochastic manner.  Consider a
query to evaluate whether there are at least $k$ agents in a given
state $s$ at a given time $t$.  Note that we cannot compile a model of
this system into a clausal form without knowing the query.  This
system can be represented as a PRISM/Px program by modeling each
agent's evolution as a Markov model.  The size of the lifted
explanation graph, and the number of recurrences for this query is
$O(k.t)$.  The recurrences are evaluated along three dimensions---
time, total number of agents, and number of agents in state $s$---
resulting in a time complexity of $O(n.k.t)$.  

%%% Local Variables: 
%%% mode: latex
%%% TeX-master: "main"
%%% End: 
o use of populations to ground instance variables. By not grounding instance
variables and using constraints on them the ground explanation tree can be
compactly represented as shown in Fig \ref{fig:lexpintrocompact}. The compact
structure can be ``unrolled'' to give the ground explanation tree by grounding
the instance variable of the root with the first available instance and
replacing all ``0'' children of the root with copies of the compact tree, with
additional constraints to exclude the first instance. The exit branch of the
root is discarded. When there are no satisfying instances to ground the root,
then the ``0'' children are replaced by the subtree of the exit branch. This
process is carried out for the remaining nodes as well. One step of this
expansion is shown in Fig \ref{fig:lexpintrocompactexpand}. The \emph{lifted explanation graphs}
and the various operations on them are defined formally in the rest of this
section.
\begin{figure}
\centering
\begin{lstlisting}[language=prolog, mathescape=true]
q :-
    $\langle X \in Coins \rangle$,
    msw(toss, X, heads),
    $\langle Y \in Coins \rangle$,
    $\langle X < Y \rangle$,
    msw(toss, Y, heads).
\end{lstlisting}
\caption{Px program}
\label{fig:lexpintropx}
\end{figure}
\begin{figure}
\begin{minipage}[b]{0.45\textwidth}
\centering
\begin{adjustbox}{scale=0.8}
\begin{tikzpicture}
  \tikzstyle{vertex}=[rectangle,fill=blue!20,draw]
  \tikzstyle{edge} = [draw, thick, ->]
  \node[vertex] (l1n1) at (0,0) {$(toss,1)$};
  \node[vertex] (l2n1) at (-3,-2) {$(toss,2)$};
  \node[vertex] (l2n2) at (3,-2) {$(toss,2)$};
  \node (l3n1) at (-4,-4) {1};
  \node[vertex] (l3n2) at (-2,-4) {$(toss,3)$};
  \node[vertex] (l3n3) at (2,-4) {$(toss,3)$};
  \node[vertex] (l3n4) at (4,-4) {$(toss,3)$};
  \node (l4n1) at (-3,-6) {1};
  \node[vertex] (l4n2) at (-1,-6) {$(toss,4)$};
  \node (l4n3) at (1,-6) {1};
  \node[vertex] (l4n4) at (3,-6) {$(toss,4)$};

  \node (l5n1) at (-2,-7) {};
  \node (l5n2) at (0, -7) {};
  \node (l5n3) at (2,-7) {};
  \node (l5n4) at (4,-7) {};

  \node (l4n5) at (3,-5) {};
  \node (l4n6) at (5,-5) {};

  \path[edge] (l1n1) -- node[left, fill=yellow] {h} (l2n1);
  \path[edge] (l1n1) -- node[right, fill=yellow] {t} (l2n2);
  \path[edge] (l2n1) -- node[left, fill=yellow] {h} (l3n1);
  \path[edge] (l2n1) -- node[right, fill=yellow] {t} (l3n2);
  \path[edge] (l2n2) -- node[left, fill=yellow] {h} (l3n3);
  \path[edge] (l2n2) -- node[right, fill=yellow] {t} (l3n4);
  \path[edge] (l3n2) -- node[left, fill=yellow] {h} (l4n1);
  \path[edge] (l3n2) -- node[right, fill=yellow] {t} (l4n2);
  \path[edge] (l3n3) -- node[left, fill=yellow] {h} (l4n3);
  \path[edge] (l3n3) -- node[right, fill=yellow] {t} (l4n4);

  \path[draw,dashed] (l3n4) -- (l4n5);
  \path[draw,dashed] (l3n4) -- (l4n6);
  \path[draw,dashed] (l4n2) -- (l5n1);
  \path[draw,dashed] (l4n2) -- (l5n2);
  \path[draw,dashed] (l4n4) -- (l5n3);
  \path[draw,dashed] (l4n4) -- (l5n4);
\end{tikzpicture}
\end{adjustbox}
\caption{Ground explanation tree}
\label{fig:lexpintroground}
\end{minipage}
\quad \quad \quad \quad
\begin{minipage}[b]{0.45\textwidth}
\centering
\begin{adjustbox}{scale=0.6}
\begin{tikzpicture}
  \tikzstyle{vertex}=[rectangle,fill=blue!20,draw]
  \tikzstyle{edge} = [draw, thick, ->]
  \node[vertex] (l1n1) at (0,0) {$(\exists X\{X \in Coins\}, (toss,X))$};
  \node[vertex] (l2n1) at (-3,-2) {$(\exists Y\{Y \in Coins, X < Y\}, (toss,Y))$};
  \node (l2n2) at (0,-2) {0};
  \node (l2n3) at (3,-2) {0};
  \node (l3n1) at (-5,-4) {1};
  \node (l3n2) at (-3,-4) {0};
  \node (l3n3) at (-1,-4) {0};
  \path[edge] (l1n1) -- node[left,fill=yellow] {h} (l2n1);
  \path[edge] (l1n1) -- node[right,fill=yellow] {t} (l2n2);
  \path[edge] (l1n1) -- node[right,fill=yellow] {exit} (l2n3);
  \path[edge] (l2n1) -- node[left,fill=yellow] {h} (l3n1);
  \path[edge] (l2n1) -- node[right,fill=yellow] {t} (l3n2);
  \path[edge] (l2n1) -- node[right,fill=yellow] {exit} (l3n3);
\end{tikzpicture}
\end{adjustbox}
\caption{Compact representation}
\label{fig:lexpintrocompact}
\vfill
\centering
\begin{adjustbox}{scale=0.6}
\begin{tikzpicture}
  \tikzstyle{vertex}=[rectangle,fill=blue!20,draw]
  \tikzstyle{edge} = [draw, thick, ->]
  \node[vertex] (l1n1) at (0,0) {$(toss,1)$};
  \node[vertex] (l2n1) at (-3,-2) {$(\exists Y\{Y \in Coins, 1 < Y\}, (toss,Y))$};
  \node[vertex] (l2n2) at (3,-2) {$(\exists X\{X \in Coins, 1 < X\}, (toss,X))$};
  \node (l3n1) at (-5,-4) {1};
  \node (l3n2) at (-3,-4) {0};
  \node (l3n3) at (-1,-4) {0};
  \node[vertex] (l3n4) at (1.5,-4) {$(\exists Y\{Y \in Coins, X < Y\}, (toss,Y))$};
  \node (l3n5) at (4,-4) {0};
  \node (l3n6) at (5,-4) {0};
  \node (l4n1) at (0,-6) {1};
  \node (l4n2) at (1.5,-6) {0};
  \node (l4n3) at (3,-6) {0};
  \path[edge] (l1n1) -- node[left,fill=yellow] {h} (l2n1);
  \path[edge] (l1n1) -- node[right,fill=yellow] {t} (l2n2);
  \path[edge] (l2n1) -- node[left,fill=yellow] {h} (l3n1);
  \path[edge] (l2n1) -- node[right,fill=yellow] {t} (l3n2);
  \path[edge] (l2n1) -- node[right,fill=yellow] {exit} (l3n3);
  \path[edge] (l2n2) -- node[left,fill=yellow] {h} (l3n4);
  \path[edge] (l2n2) -- node[left,fill=yellow] {t} (l3n5);
  \path[edge] (l2n2) -- node[right,fill=yellow] {exit} (l3n6);
  \path[edge] (l3n4) -- node[left,fill=yellow] {h} (l4n1);
  \path[edge] (l3n4) -- node[left,fill=yellow] {t} (l4n2);
  \path[edge] (l3n4) -- node[right,fill=yellow] {exit} (l4n3);
\end{tikzpicture}
\end{adjustbox}
\caption{Expansion of compact representation}
\label{fig:lexpintrocompactexpand}
\end{minipage}
\end{figure}
}
%%% Local Variables: 
%%% mode: latex
%%% TeX-master: "main"
%%% End: 

\section{Lifted Inference using Lifted
  Explanations}\label{sec:inference}
In this section we describe a technique to compute answer probabilities in a
lifted fashion from closed lifted explanation graphs. This technique works on a
restricted class of lifted explanation graphs satisfying a property we call the
\emph{ frontier subsumption property}. 
\begin{Def}[Frontier]
  Given a closed lifted explanation graph $(\Omega:\eta,\psi)$, the frontier of
  $\psi$ w.r.t $X \in \Omega$ denoted $\id{frontier}_X(\psi)$ is the set of
  non-zero maximal subtrees of $\psi$, which do not contain a node with $X$ as
  the instance variable.
\end{Def}
Analogous to the set representation of explanations described in
\ref{sec:pxandinf}, we consider the set representations of lifted explanations,
i.e., root-to-leaf paths in the DAGs of lifted explanation graphs that end in a
``1'' leaf.

We consider \emph{term substitutions} that can be applied to lifted explanations.
These substitutions replace a variable by a term and further apply standard 
re-writing rules such as simplification of algebraic expressions. As before, we
allow \emph{term mappings} that specify a set of \emph{term substitutions}.
\begin{Def}[Frontier subsumption property]
  A closed lifted explanation graph $(\Omega:\eta,\psi)$ satisfies the frontier
  subsumption property w.r.t $X \in \Omega$, if under term mappings
  $\sigma_1 = \{X \pm k + 1 / Y \mid \langle X \pm k < Y \rangle \in \eta\}$ and
  $\sigma_2 = \{X + 1 / X\}$, every tree $\phi \in \id{frontier}_X(\psi)$
  satisfies the following condition: for every lifted explanation $E_2$ in
  $\psi$, there is a lifted explanation $E_1$ in $\phi$ such that $E_1\sigma_1$ is 
  a sub-explanation (i.e., subset) of $E_2 \sigma_2$.
\end{Def}
A lifted explanation graph is said to satisfy frontier subsumption property, if
it is satisfied for each bound variable. This property can be checked in a
bottom up fashion for all bound variables in the graph. The tree obtained by
replacing all subtrees in $\id{frontier}_X(\psi)$ by $1$ in $\psi$ is denoted
$\widehat{\psi}_X$.

\myparagraph{Examples for frontier subsumption property.}
Note that the lifted explanation graph in Figure \ref{fig:intro-ex}(c) satisfies
the frontier subsumption property. The frontier w.r.t the instance variable $X$
contains the left subtree rooted at $(\mathtt{toss},Y)$. The single lifted
explanation in $\psi$ is $\{toss[X]=h,toss[Y]=h\}$. The single explanation in
the frontier tree is $\{toss[Y]=h\}$. Under the mappings $\sigma_1=\{X+1/Y\}$
and $\sigma_2=\{X+1/X\}$. We can see that $\{toss[Y]=h\}\sigma_1$ is a
sub-explanation of $\{toss[X]=h,toss[Y]=h\}\sigma_2$. This property is not
satisfied however for the dice example. The frontier w.r.t $X''$ (see Figure
\ref{fig:dicegraph2}) contains the two subtrees rooted at $(\mathtt{roll},Y)$
and $(\mathtt{roll},Y')$. Now there are two lifted explanations in $\psi$:
$\{roll[X'']=1,roll[Y]=1\}$ and $\{roll[X'']=2,roll[Y']=2\}$. Let us consider
the first tree in the frontier. It contains a single lifted explanation
$\{roll[Y]=1\}$. Under the mappings $\sigma_1=\{X''+1/Y\}$ and
$\sigma_2=\{X''+1/X''\}$, $\{roll[Y]=1\}\sigma_1$ is not a sub-explanation of
$\{roll[X'']=2,roll[Y']=2\}\sigma_2$.

For closed lifted explanation graphs satisfying the frontier subsumption property, the
probability of query answers can be computed using the following set of
recurrences. With each subtree $\psi = (s,t)[\alpha_i:\psi_i]$ of the DAG of the
lifted explanation graph, we associate the function $f(\sigma,\psi)$ where
$\sigma$ is a (possibly incomplete) mapping of variables in $\Omega$ to values
in their domains.
\begin{Def}[Probability recurrences]\label{def:recurrences}
  Given a closed lifted explanation graph $(\Omega:\eta,\psi)$, we define
  $f(\sigma, \psi)$ (as well as $g(\sigma,\psi)$ and $h(\sigma,\psi)$ wherever
  applicable) for a partial mapping $\sigma$ of variables in $\Omega$ to values
  in their domains based on the structure of $\psi$. As before
  $\psi=(s,t)[\alpha_i:\psi_i]$
\begin{description}
  \item [Case 1:] $\psi$ is a $0$ leaf node. Then
    $
      f(\sigma, 0) = 0
    $
  \item [Case 2:] $\psi$ is a $1$ leaf node. Then
    $
      f(\sigma, 1) = 
      \begin{cases}
        1, \mbox{ if } \sem{\eta\sigma} \neq \emptyset\\
        0, \mbox{ otherwise}
      \end{cases}
    $
  \item [Case 3:] $t\sigma$ is a constant. Then
    $
      f(\sigma, \psi) = 
      \begin{cases}
        \sum_{\alpha_i \in D_s} \pi_s(\alpha_i) \cdot f(\sigma, \psi_i),\mbox{
          if } \sem{\eta\sigma} \neq \emptyset\\
        0, \mbox{ otherwise}
      \end{cases}
    $
  \item [Case 4:] $t\sigma \in \Omega$, and
    $range(t,\eta\sigma) = (l,u)$. Then
    \begin{align*}
      f(\sigma, \psi) = &
                          \begin{cases}
                            h(\sigma[l/t],\psi), \mbox{ if } \sem{\eta\sigma} \neq \emptyset\\
                            0, \mbox{ otherwise}
                          \end{cases}\\
      h(\sigma[c/t],\psi) = & 
                              \begin{cases}
                                % g(\sigma[c/t],\psi) + ((1 -
                                % P(\widehat{\phi}_{(s,c)})) \times
                                % h(\sigma[c+1/t],\psi)), \mbox{ if } c < u\\
                                g(\sigma[c/t],\psi) + ((1 -
                                P(\widehat{\psi}_{X})) \times
                                h(\sigma[c+1/t],\psi)), \mbox{ if } c < u\\
                                g(\sigma[c/t],\psi), \mbox{ if } c = u
                              \end{cases}\\
      g(\sigma,\psi) = & 
                              \begin{cases}
                                \sum_{\alpha_i \in D_s} \pi_s(\alpha_i) \cdot
                                f(\sigma, \psi_i), \mbox{ if } \sem{\eta\sigma} \neq \emptyset\\
                                0, \mbox{ otherwise}
                              \end{cases}
    \end{align*}
\end{description}
\end{Def}
In the above definition $\sigma[c/t]$ refers to a new partial mapping obtained
by augmenting $\sigma$ with the substitution $[c/t]$,
$P(\widehat{\psi}_t)$ 
is the sum of the probabilities of all branches
leading to a $1$ leaf in 
$\widehat{\psi}_t$.
The functions $f$, $g$ and $h$ defined above can be readily specialized for
each $\psi$.  Moreover, the parameter $\sigma$ can be replaced by the
tuple of values actually used by a function.  These rewriting
steps yield recurrences such as those shown in
Fig.~\ref{fig:intro-ex-eqn}.  Note that $P(\widehat{\psi}_t)$ can be
computed using recurrences as well (shown as $\widehat{f}$ in 
Fig.~\ref{fig:intro-ex-eqn}).

\begin{Def}[Probability of Lifted Explanation
  Graph]\label{def:prob-lifted-expl}
Let $(\Omega:\eta,\psi)$ be a closed lifted explanation graph.  Then, the
probability of explanations represented by the graph,  $\id{prob}((\Omega:\eta,\psi))$,
is the value of $f(\{\},\psi)$. 
\end{Def}

\begin{restatable}[Correctness of Lifted Inference]{Thm}{infcorrect}
\label{thm:inference-correctness}
  Let $(\Omega:\eta,\psi)$ be a closed lifted explanation graph, and
  $\phi = Gr(\Omega:\eta,\psi)$ be the corresponding ground explanation graph.
  Then $\id{prob}((\Omega:\eta,\psi)) = \id{prob}(\phi)$.
\end{restatable}

Given a closed lifted explanation graph, let $k$ be the maximum number of
instance variables along any root to leaf path. Then the function
$f(\sigma,\psi)$ for the leaf will have to be computed for each mapping of the
$k$ variables. Each recurrence equation itself is either of constant size or
bounded by the number of children of a node.  Using dynamic programming
(possibly implemented via tabling), a solution to the recurrence equations can
be computed in polynomial time.

\begin{Thm}[Efficiency of Lifted Inference]\label{thm:inference-efficiency}
  Let $\psi$ be a closed lifted inference graph, $n$ be the size of the largest
  population, and $k$ be the largest number of instance variables along any root
  of leaf path in $\psi$. Then, $f(\{\},\psi)$ can be computed in
  $O(|\psi| \times n^{k})$ time.
\end{Thm}

There are two sources of further optimization in the generation and evaluation
of recurrences. First, certain recurrences may be transformed into closed form
formulae which can be more efficiently evaluated. For instance, the closed form
formula for $h(\sigma,\psi)$ for the subtree rooted at the node $(toss,Y)$ in
Fig~\ref{fig:intro-ex}(c) can be evaluated in $O(\log(n))$ time while a naive
evaluation of the recurrence takes $O(n)$ time. Second, certain functions
$f(\sigma,\psi)$ need not be evaluated for every mapping $\sigma$ because they
may be independent of certain variables. For example, leaves are always
independent of the mapping $\sigma$.

\myparagraph{Other Examples.} 
There are a number of simple probabilistic models that cannot be
tackled by other lifted inference techniques but can be encoded
in Param-Px and solved using our technique.  For one such example,
consider an urn with $n$ balls, where the color of each ball is given
by a distribution.  Determining the probability that there are at
least two green balls is easy to phrase as a directed first-order
graphical model.  However, lifted inference over such models can no
longer be applied if we need to determine the probability of at least
two green or two red balls. The probability computation for one of these events
can be viewed as a generalization of noisy-OR probability computation, however
dealing with the union requires the handling of intersection of
the two events, due to which the $O(\log(N))$ time computation is no
longer feasible. 

For a more complex example, we use an instance of a \emph{collective
graphical model}~\cite{NIPS2011_4220}.  
In particular, consider a system of $n$ agents where each
agent moves between various states in a stochastic manner.  Consider a
query to evaluate whether there are at least $k$ agents in a given
state $s$ at a given time $t$.  Note that we cannot compile a model of
this system into a clausal form without knowing the query.  This
system can be represented as a PRISM/Px program by modeling each
agent's evolution as a Markov model.  The size of the lifted
explanation graph, and the number of recurrences for this query is
$O(k.t)$.  When the recurrences are evaluated along three dimensions:
time, total number of agents, and number of agents in state $s$,
resulting in a time complexity of $O(n.k.t)$.  

%%% Local Variables: 
%%% mode: latex
%%% TeX-master: "main"
%%% End: 

\section{Related Work and Discussion}\label{sec:related}
First-order graphical models \cite{poole2003first,braz2005lifted} are compact
representations of propositional graphical models over populations. The key
concepts in this field are that of \emph{parameterized random variables} and
\emph{parfactors}. A parameterized random variable stands for a population of
i.i.d. propositional random variables (obtained by grounding the logical
variables). Parfactors are factors (potential functions) on parameterized random
variables. By allowing large number of identical factors to be specified in a
first-order fashion, first-order graphical models provide a representation that
is independent of the population size. A key problem, then, is to
perform \emph{lifted}  probabilistic
inference over these models, i.e. without grounding the factors
unnecessarily.  The earliest such technique was \emph{inversion
  elimination} due to~\citeN{poole2003first}. When summing out a
parameterized random variable (i.e., all its groundings), it is observed that if
all the logical variables in a parfactor are contained in the parameterized
random variable, it can be summed out without grounding the parfactor.

The idea of \emph{inversion elimination}, though powerful, exploits
one of the many forms of symmetry present in first-order graphical
models.  Another kind of symmetry present in such models 
is that the values of an intermediate factor may depend on the
histogram of propositional random variable outcomes, rather than their exact
assignment. This symmetry is exploited by \emph{counting elimination}
\cite{braz2005lifted} and elimination by \emph{counting formulas}
\cite{milch2008lifted}.

\citeN{van2011lifted} presented a form of lifted inference that 
uses constrained CNF theories with positive and
negative weight functions over predicates as input. Here the task of
probabilistic inference in transformed to one of weighted model
counting. To do the latter, the CNF theory is compiled into a
structure known as first-order deterministic decomposable negation normal form.
The compiled representation allows lifted inference by
avoiding grounding of the input theory.
This technique is applicable so long as
the model can be formulated as a constrained CNF theory.

\citeN{bellodi2014lifted} present another approach to lifted
inference 
for probabilistic logic programs.
The idea is to convert a ProbLog program to
parfactor representation and use a modified version of generalized counting first
order variable elimination algorithm \cite{taghipour2013lifted} to perform
lifted inference.  Problems where the model size is dependent on the
query, such as models with temporal aspects, are difficult to solve
with the knowledge compilation approach.

In this paper, we presented a technique for lifted inference in probabilistic
logic programs using lifted explanation graphs.  This technique is a
natural generalization of inference techniques based on ground
explanation graphs, and follows the two step approach: generation of
an explanation graph, and a subsequent traversal to compute
probabilities.   While the size of the lifted explanation graph is often
independent of population, computation of probabilities may take
time that is polynomial in the size of the population.  A more
sophisticated approach to computing probabilities from lifted
explanation graph, by generating closed form formulae where possible,
will enable efficient inference.   Another direction of research would
be to generate hints for lifted inference based on program constructs
such as aggregation operators.  Finally, our future work is focused on
performing lifted inference over probabilistic logic programs that
represent undirected and discriminative models.

\subparagraph{Acknowledgments.}  This work was supported in part 
    by NSF grants IIS 1447549  and CNS 1405641. 
We thank Andrey Gorlin for discussions and review of this work.  

%%% Local Variables: 
%%% mode: latex
%%% TeX-master: "main"
%%% End: 

%\input{disc}

\appendix
\section{Proofs}

\andorcorrect*
\begin{proof}
  The proof is by structural induction.
  \begin{description}

    \item [Case 1:] $Q(\Omega,\eta) \land Q(\Omega',\eta')$ is unsatisfiable.
      \begin{description}
      \item[Case 1.1:] $\land$ operation
        
        $Q(\Omega,\eta) \land Q(\Omega,\eta')$ is unsatisfiable, implies that
        for each $\sigma \in \Sigma$, either one or both of
        $Q(\Omega,\eta)\sigma$ and $Q(\Omega',\eta')\sigma$ are
        unsatisfiable. Either one or both
        of $\exists\Omega.\eta\sigma$ and $\exists\Omega'.\eta'\sigma$ are
        unsatisfiable. By definition \ref{def:substitution}, either one or both
        of the $(\Omega:\eta,\psi)\sigma$ and $(\Omega':\eta',\psi')\sigma$
        should be $(\emptyset:\{\texttt{false}\},0)$. By definition
        \ref{def:grounding}, we know that
        $Gr((\emptyset:\{\texttt{false}\},0)) = 0$. Therefore
        $Gr((\Omega:\eta,\psi)\sigma) \land Gr((\Omega':\eta',\psi')\sigma) =
        0$. The $\land$ operation in this case is defined as
        \[
          (\Omega:\eta,\psi) \land (\Omega':\eta',\psi') \rightarrow
          (\emptyset:\{\texttt{false}\},0)
        \]
        Therefore,
        $Gr(((\Omega:\eta,\psi) \land (\Omega':\eta',\psi'))\sigma) =
        Gr((\Omega:\eta,\psi)\sigma) \land Gr((\Omega':\eta',\psi')\sigma)$.

      \item[Case 1.2:] $\lor$ operation

        Assume without loss of generality that
        $(\Omega':\eta',\psi')\sigma = (\emptyset:\{\texttt{false}\},0)$. Then
        by definition \ref{def:grounding}
        $Gr((\Omega:\eta,\psi)\sigma) \lor Gr((\Omega':\eta',\psi')\sigma) =
        Gr((\Omega:\eta,\psi)\sigma)$. Since the definition of $\lor$ operation
        in this case is to backtrack and return both $(\Omega:\eta,\psi)$ and
        $(\Omega':\eta',\psi')$, under the substitution $\sigma$,
        $(\Omega':\eta',\psi')$ will be discarded.  Therefore,
        $Gr(((\Omega:\eta,\psi) \lor (\Omega':\eta',\psi'))\sigma) =
        Gr((\Omega:\eta,\psi)\sigma) \lor Gr((\Omega':\eta',\psi')\sigma)$.
      \end{description}

    \item[Case 2:] $Q(\Omega,\eta) \land Q(\Omega',\eta')$ is satisfiable. 
      \begin{description}
      \item[Case 2.1:] $\land$ operation

        The $\land$ operation is defined as 
          $
            (\Omega:\eta,\psi) \land (\Omega':\eta',\psi') \rightarrow (\Omega
            \cup \Omega': \eta \land \eta', \psi \land \psi')
          $
          \begin{description}
            \item[Case 2.1.1:] $(Q(\Omega,\eta) \land Q(\Omega',\eta'))\sigma$
              is unsatisfiable

              $(Q(\Omega,\eta)\land Q(\Omega',\eta'))\sigma$ is unsatisfiable,
              implies atleast one of $Q(\Omega,\eta)\sigma$ and
              $Q(\Omega',\eta)\sigma$ is unsatisfiable. Atleast one of
              $\exists\Omega.\eta\sigma$ and $\exists\Omega'.\eta'\sigma$ is
              unsatisfiable. By definition \ref{def:substitution}, atleast one
              of $(\Omega:\eta,\psi)\sigma$ and $(\Omega':\eta',\psi')\sigma$ is
              equal to $(\emptyset:\{\texttt{false}\},0)$. By definition
              \ref{def:grounding},
              $Gr((\Omega:\eta,\psi)\sigma) \land
              Gr((\Omega':\eta',\psi')\sigma) = 0$. Further, if one of
              $\exists\Omega.\eta\sigma$ and $\exists\Omega'.\eta'\sigma$ is
              unsatisfiable $\exists \Omega\cup\Omega'. \eta \land \eta' \sigma$
              is also unsatisfiable. Therefore, by definition
              \ref{def:substitution},
              $Gr((\Omega\cup\Omega':\eta \land \eta', \psi \land \psi')\sigma)
              = 0$. Therefore
              $Gr(((\Omega:\eta,\psi) \land (\Omega':\eta',\psi'))\sigma) =
              Gr((\Omega:\eta,\psi)\sigma) \land
              Gr((\Omega':\eta',\psi')\sigma)$.

            \item[Case 2.1.1:] $(Q(\Omega,\eta) \land Q(\Omega',\eta'))\sigma$
              is satisfiable.
              
              \begin{description}
                \item[Case 2.1.1.1:] $\psi = 0$ (analogously $\psi'=0$).

                  By definition \ref{def:grounding}, $Gr((\Omega:\eta,\psi)\sigma) \land
                  Gr((\Omega':\eta',\psi')\sigma) = 0$. Based on the definition
                  of $\land$ operation
                  $(\Omega\cup\Omega':\eta \land \eta', \psi \land \psi')\sigma =
                  (\Omega\cup\Omega':\eta \land \eta', 0)\sigma$. By definition
                  \ref{def:grounding}
                  $Gr((\Omega\cup\Omega':\eta\land\eta',0)\sigma) = 0$. Therefore,
                  $Gr(((\Omega:\eta,\psi) \land (\Omega':\eta',\psi'))\sigma) =
                  Gr((\Omega:\eta,\psi)\sigma) \land
                  Gr((\Omega':\eta',\psi')\sigma)$.

                \item[Case 2.1.1.2:] $\psi = 1$ (analogously $\psi'=1$).
                  
                  By definition \ref{def:grounding},
                  $Gr((\Omega:\eta,\psi)\sigma) = 1$. Therefore
                  $Gr((\Omega:\eta,\psi)\sigma) \land
                  Gr((\Omega':\eta',\psi')\sigma) =
                  Gr((\Omega':\eta',\psi')\sigma)$.
                  % Since $(Q(\Omega,\eta) \land Q(\Omega',\eta'))\sigma$ is
                  % satisfiable, it implies that $\exists \Omega.\eta\sigma$ and
                  % $\exists \Omega'.\eta'\sigma$ are both satisfiable.
                  Based on the definition of $\land$ operation
                  $(\Omega\cup\Omega':\eta\land\eta', \psi\land\psi') =
                  (\Omega\cup\Omega':\eta\land\eta',\psi')$. By definition
                  \ref{def:substitution}
                  $(\Omega\cup\Omega':\eta\land\eta',\psi')\sigma =
                  (\Omega\cup\Omega':(\eta\land\eta')\sigma,
                  \psi'\sigma)$ and $(\Omega':\eta',\psi')\sigma =
                  (\Omega':\eta'\sigma, \psi'\sigma)$. We can claim that
                  $Gr((\Omega':\eta'\sigma, \psi'\sigma)) =
                  Gr((\Omega\cup\Omega':(\eta\land\eta')\sigma, \psi'\sigma))$
                  because, $range(t,\eta'\sigma) =
                  range(t,(\eta\land\eta')\sigma)$ for any $t \in \Omega'$ Why?
                  Because there exist no variables in common between
                  $\eta\sigma$ and $\eta'\sigma$ and $\eta\sigma \land
                  \eta'\sigma$ is satisfiable based on the assumptions.

                \item[Case 2.1.1.3]: Neither $\psi$ nor $\psi'$ is a leaf node
                  and $(s,t) < (s',t')$ (analogously $(s',t') < (s,t)$).

                  Since $(Q(\Omega,\eta) \land Q(\Omega',\eta')\sigma)$ is
                  satisfiable, we can conclude that $\eta\sigma$ and
                  $\eta'\sigma$ are satisfiable. By definition
                  \ref{def:substitution},
                  $Gr((\Omega:\eta,\psi)\sigma) \land
                  Gr((\Omega':\eta',\psi')\sigma) =
                  Gr(\Omega:\eta\sigma,\psi\sigma) \land
                  Gr(\Omega':\eta'\sigma,\psi'\sigma)$. Let us consider the case
                  where $t$ is either a constant or a free variable. Since
                  $(s,t) < (s',t')$ it implies that
                  $(s,t\sigma) \prec (s',t'\sigma)$. By definition
                  \ref{def:grounding},
                  $Gr(\Omega:\eta\sigma, \psi\sigma) \land
                  Gr(\Omega':\eta'\sigma, \psi'\sigma) =
                  Gr(\Omega,\eta\sigma,\psi\sigma) \land
                  Gr(\Omega',\eta'\sigma,\psi'\sigma)$. Further
                  $Gr(\Omega:\eta\sigma, \psi\sigma) \land
                  Gr(\Omega':\eta'\sigma, \psi'\sigma) =
                  (s,t\sigma)[\alpha_i:Gr(\Omega,\eta\sigma,\psi_i\sigma) \land
                  Gr(\Omega',\eta'\sigma,\psi'\sigma)]$. Based on the definition
                  of $\land$ operation
                  $(\Omega\cup\Omega':\eta\land\eta',\psi\land\psi')\sigma =
                  (\Omega\cup\Omega':\eta\cup\eta', (s,t)[\alpha_i:\psi_i \land
                  \psi'])\sigma$. Therefore,
                  $Gr((\Omega\cup\Omega':\eta\land\eta',\psi\land\psi')\sigma) =
                  Gr(\Omega\cup\Omega', (\eta\land\eta')\sigma,
                  (s,t)[\alpha_i:\psi_i \land \psi']\sigma) =
                  (s,t\sigma)[\alpha_i:Gr(\Omega\cup\Omega',
                  (\eta\land\eta')\sigma, (\psi_i \land \psi')\sigma)]$. Based
                  on inductive hypothesis,
                  $Gr(\Omega,\eta\sigma,\psi_i\sigma) \land
                  Gr(\Omega',\eta'\sigma,\psi'\sigma) =
                  Gr(\Omega\cup\Omega',(\eta\land\eta')\sigma,
                  (\psi_i\land\psi')\sigma)$. Now consider the case when
                  $t \in \Omega$. By definition \ref{def:grounding}
                  $Gr(\Omega:\eta,\psi)\sigma \land
                  Gr(\Omega':\eta',\psi')\sigma =
                  Gr(\Omega,\eta\sigma,\psi\sigma) \land
                  Gr(\Omega',\eta'\sigma,\psi'\sigma) = \bigvee_{c \in
                    range(t,\eta\sigma)}
                  (s,c)[\alpha_i:Gr(\Omega\setminus\{t\},\eta\sigma[c/t],\psi_i\sigma[c/t])
                  \land Gr(\Omega,\eta'\sigma,\psi'\sigma)]$. Similarly,
                  $Gr(\Omega\cup\Omega', (\eta\land\eta')\sigma,
                  (s,t)[\alpha_i:\psi_i\land\psi']\sigma) = \bigvee_{c \in
                    range(t,(\eta\land\eta')\sigma)}
                  (s,c)[\alpha_i:Gr(\Omega\cup\Omega'\setminus\{t\},(\eta\land\eta')\sigma[c/t],
                  (\psi_i\land\psi')\sigma[c/t])]$. But $range(t,\eta\sigma) =
                  range(t,(\eta\land\eta')\sigma)$ and based on inductive
                  hypothesis, $Gr(\Omega\setminus\{t\},
                  \eta\sigma[c/t],\psi_i\sigma[c/t]) \land
                  Gr(\Omega',\eta'\sigma,\psi'\sigma) =
                  Gr(\Omega\cup\Omega'\setminus\{t\},
                  (\eta\land\eta')\sigma[c/t], (\psi_i\land\psi')\sigma[c/t])$

                \item[Case 2.1.1.4:] Neither $\psi$ nor $\psi'$ is a leaf node
                  and $(s,t)=(s',t')$.

                  Since the variables in the lifted explanation graphs are
                  standardized apart, this implies that neither $t$ nor $t'$ is
                  a bound variable. $Gr((\Omega:\eta,\psi)\sigma) \land
                  Gr((\Omega':\eta',\psi')\sigma) =
                  (s,t\sigma)[\alpha_i:Gr(\Omega,\eta\sigma,\psi_i\sigma) \land
                  Gr(\Omega',\eta'\sigma, \psi'_i\sigma)]$. Similarly,
                  $Gr(\Omega\cup\Omega', (\eta\land\eta')\sigma,
                  (s,t)[\alpha_i:\psi_i\land\psi'_i]\sigma) =
                  (s,t\sigma)[\alpha_i:Gr(\Omega\cup\Omega',
                  (\eta\land\eta')\sigma, (\psi_i\land\psi'_i)\sigma)]$. Based
                  on inductive hypothesis, $Gr(\Omega,\eta\sigma,\psi_i\sigma)
                  \land Gr(\Omega', \eta'\sigma, \psi'_i\sigma) =
                  Gr(\Omega\cup\Omega', (\eta\land\eta')\sigma,
                  (\psi_i\land\psi'_i)\sigma)$.

                \item[Case 2.1.1.5:] Neither $\psi$ nor $\psi'$ is a leaf node
                  and $(s,t) \not \sim (s',t')$ and $t$ is a free variable or a
                  constant and $t'$ is a free variable.

                  Consider the case where $t$ and $t'$ are free variables, or
                  $t$ is a constant and $t'$ is a free variable. Based
                  on $\sigma$ exactly one of $(s,t\sigma) < (s,t'\sigma)$,
                  $(s,t\sigma)=(s',t\sigma)$ and $(s,t'\sigma) < (s,t\sigma)$
                  will hold. According to the definition of $\land$ operation,
                  three lifted explanation graphs are returned
                  $(\Omega\cup\Omega':\eta\land\eta'\land t<t',
                  \psi\land\psi')$, $(\Omega\cup\Omega':\eta\land\eta'\land t=t',
                  \psi\land\psi')$, and $(\Omega\cup\Omega':\eta\land\eta'\land t'<t,
                  \psi\land\psi')$. Under the substitution $\sigma$ only one
                  will be retained. And the proof then proceeds as in case
                  2.1.1.3 or case 2.1.1.4.

                \item[Case 2.1.1.6:] Neither $\psi$ nor $\psi'$ is a leaf node
                  and $(s,t) \not \sim (s',t')$ and $t$ is a free variable or a
                  constant and $t'$ is a bound variable. Consider
                  $Gr((\Omega:\eta,\psi)\sigma) \land
                  Gr((\Omega':\eta',\psi')\sigma)$. This can be written as
                  $(s,t\sigma)[\alpha_i:Gr(\Omega,\eta\sigma, \psi_i\sigma)] \land
                  \bigvee_{c \in range(t',\eta'\sigma)} (s,c)[\alpha'_i:
                  Gr(\Omega'\setminus\{t'\},\eta'\sigma[c/t'],\psi'_i\sigma[c/t'])]$. By
                  using continuity of range we can rewrite it as
                  \begin{align*}
                    &
                      (s,t\sigma)[\alpha_i:Gr(\Omega,\eta\sigma,\psi_i\sigma)]\land \\
                    & \bigvee_{c \in range(t',(\eta'\land t < t')\sigma)}(s,c)[\alpha'_i:
                      Gr(\Omega'\setminus\{t'\},\eta'\sigma[c/t'],\psi'_i\sigma[c/t'])]\\
                    \vee
                    &
                      (s,t\sigma)[\alpha'_i:Gr(\Omega'\setminus\{t'\},\eta'\sigma[t\sigma/t'],
                      \psi'_i\sigma[t\sigma/t'])]\\
                    & \vee \bigvee_{c \in range(t',(\eta'\land t' < t)\sigma)}(s,c)[\alpha'_i:
                      Gr(\Omega'\setminus\{t'\},\eta'\sigma[c/t'],\psi'_i\sigma[c/t'])]
                  \end{align*}
                  By distributivity of $\land$ over $\lor$ for ground
                  explanation graphs, we can rewrite it as 
                  \begin{align*}
                     (s,t\sigma) & [\alpha_i:Gr(\Omega,\eta\sigma,\psi_i\sigma)]\land\\
                    & \bigvee_{c \in range(t',(\eta'\land t < t')\sigma)}(s,c)[\alpha'_i:
                      Gr(\Omega'\setminus\{t'\},\eta'\sigma[c/t'],\psi'_i\sigma[c/t'])]\\
                    \vee
                     (s,t\sigma) & [\alpha_i:Gr(\Omega,\eta\sigma,\psi_i\sigma)]\land
                      (s,t\sigma) [\alpha'_i:Gr(\Omega'\setminus\{t'\},\eta'\sigma[t\sigma/t'],
                      \psi'_i\sigma[t\sigma/t'])]\\
                    \vee
                     (s,t\sigma) & [\alpha_i:Gr(\Omega,\eta\sigma,\psi_i\sigma)]\\
                    & \bigvee_{c \in range(t',(\eta'\land t' < t)\sigma)}(s,c)[\alpha'_i:
                      Gr(\Omega'\setminus\{t'\},\eta'\sigma[c/t'],\psi'_i\sigma[c/t'])]
                  \end{align*}
                  This can be re-written as 
                  \begin{align*}
                    Gr((\Omega:\eta,\psi)\sigma) \land Gr((\Omega':\eta'\land t<
                    t', \psi')\sigma) \\
                    \vee Gr((\Omega:\eta,\psi)\sigma) \land Gr((\Omega':\eta'
                    \land t=t', \psi')\sigma) \\
                    Gr((\Omega:\eta,\psi)\sigma) \land Gr((\Omega':\eta'\land
                    t'<t, \psi')\sigma)
                  \end{align*}
                  By inductive hypothesis this is equal to 
                  \[
                    \bigvee_{\varphi \in \{t<t',t=t',t'<t'\}}
                    Gr((\Omega\cup\Omega':\eta\land\eta'\land\varphi, \psi \land \psi')\sigma)
                  \]

                \item[Case 2.1.1.7:] Neither $\psi$ nor $\psi'$ is a leaf node
                  and $(s,t) \not \sim (s',t')$ and $t,t'$ are bound variables.

                  When $l_1=l_2$ and $u_1=u_2$, $Gr((\Omega:\eta,\psi)\sigma)
                  \land Gr((\Omega':\eta',\psi')\sigma)$ can be written as 
                  \[
                    \bigvee_{c \in
                      [l_1,u_1]}(s,c)[\alpha_i:Gr((\Omega,\eta,\psi_i)\sigma[c/t])]
                    \land
                    \bigvee_{c' \in
                      [l_2,u_2]}(s,c')[\alpha_i:Gr((\Omega',\eta',\psi'_i)\sigma[c'/t'])]
                  \]
                  Let the sequence of positive integers in the interval
                  $[l_1,u_1]$ be $\langle l_1=k_1,k_2,\ldots,k_n=u_1 \rangle$. Then the
                  above expression can be re-written as
                  \begin{align*}
                    ( (s,k_1) & [\alpha_i:Gr((\Omega,\eta,\psi_i)\sigma[k_1/t])]\land \\
                    & \bigvee_{c' \in [k_1,k_n]}
                    (s,c')[\alpha_i:Gr((\Omega',\eta',\psi'_i)\sigma[c/t'])] )\vee \\
                    ( (s,k_2) & [\alpha_i:Gr((\Omega,\eta,\psi_i)\sigma[k_2/t])]\land \\
                    & \bigvee_{c' \in [k_1,k_n]}
                    (s,c')[\alpha_i:Gr((\Omega',\eta',\psi'_i)\sigma[c'/t'])] )\vee\\
                    & \ldots\\
                    ( (s,k_n) & [\alpha_i:Gr((\Omega,\eta,\psi_i)\sigma[k_n/t])]\land \\
                    & \bigvee_{c' \in [k_1,k_n]}
                    (s,c')[\alpha_i:Gr((\Omega',\eta',\psi'_i)\sigma[c'/t'])] )
                  \end{align*}
                  This can again be re-written as
                  \begin{align*}
                    \bigvee_{c \in [k_1,k_n]}
                    ( (s,c) & [\alpha_i:Gr((\Omega,\eta,\psi_i)\sigma[c/t])] \land\\
                    & (s,c)[\alpha_i:Gr((\Omega',\eta',\psi'_i)\sigma[c/t'])] )\\
                    \bigvee_{d \in [k_1,k_{n-1}]}
                    ( (s,d) & [\alpha_i:Gr((\Omega,\eta,\psi_i)\sigma[d/t])] \land \\
                    & \bigvee_{e \in [d+1, k_n]}
                    (s,e)[\alpha_i:Gr((\Omega',\eta',\psi'_i)\sigma[e/t'])] )\\
                    \bigvee_{f \in [k_1,k_{n-1}]}
                    ( (s,f) & [\alpha_i:Gr((\Omega',\eta',\psi'_i)\sigma[f/t])] \land\\
                    & \bigvee_{g \in [d+1, k_n]}
                    (s,g)[\alpha_i:Gr((\Omega,\eta,\psi_i)\sigma[g/t'])] )
                  \end{align*}
                  By moving the substitutions out we get the following equivalent expression.
                  \begin{align*}
                    \bigvee_{c \in [k_1,k_n]}
                     (s,t) & [\alpha_i:Gr((\Omega,\eta,\psi_i))]\sigma[c/t] \land\\
                    & (s,t')[\alpha_i:Gr((\Omega',\eta',\psi'_i))]\sigma[c/t'] )\\
                    \bigvee_{d \in [k_1,k_{n-1}]}
                    ( (s,t) & [\alpha_i:Gr((\Omega,\eta,\psi_i))]\sigma[d/t] \land\\
                    & \bigvee_{e \in [d+1, k_n]}
                    (s,t') [\alpha_i:Gr((\Omega',\eta',\psi'_i))]\sigma[e/t'] )\\
                    \bigvee_{f \in [k_1,k_{n-1}]}
                    ( (s,t') & [\alpha_i:Gr((\Omega',\eta',\psi'_i))]\sigma[f/t'] \land\\
                    & \bigvee_{g \in [f+1, k_n]}
                    (s,t) [\alpha_i:Gr((\Omega,\eta,\psi_i))]\sigma[g/t] )
                  \end{align*}
                  Since $[d+1,k_n] = \emptyset$ and $[f+1,k_n] =\emptyset$ when
                  $d = k_n$ and $f=k_n$ respectively, the above expression can
                  be re-written as follows.
                  \begin{align*}
                    \bigvee_{c \in [k_1,k_n]}
                    ( (s,t) & [\alpha_i:Gr((\Omega,\eta,\psi_i))]\sigma[c/t] \land\\
                    & (s,t')[\alpha_i:Gr((\Omega',\eta',\psi'_i))]\sigma[c/t'] )\\
                    \bigvee_{d \in [k_1,k_n]}
                    ( (s,t) & [\alpha_i:Gr((\Omega,\eta,\psi_i))]\sigma[d/t] \land \\
                    & \bigvee_{e \in [d+1, k_n]}
                    (s,t')[\alpha_i:Gr((\Omega',\eta',\psi'_i))]\sigma[e/t'] )\\
                    \bigvee_{f \in [k_1,k_n]}
                    ( (s,t') & [\alpha_i:Gr((\Omega',\eta',\psi'_i))]\sigma[f/t'] \land\\
                    & \bigvee_{g \in [f+1, k_n]}
                    (s,t)[\alpha_i:Gr((\Omega,\eta,\psi_i))]\sigma[g/t] )
                  \end{align*}
                  First we transform the above expression by using $c$ in place of $d$ and $f$.
                  Further, we perform a simple renaming operation with a new variable $t''$ to
                  get the following equivalent expression.
                  \begin{align*}
                    \bigvee_{c \in [k_1,k_n]}
                    ( (s,t'') & [\alpha_i:Gr((\Omega,\eta,\psi_i))]\sigma[t''/t][c/t''] \land\\
                    & (s,t'')[\alpha_i:Gr((\Omega',\eta',\psi'_i))]\sigma[t''/t'][c/t''] )\\
                    \bigvee_{c \in [k_1,k_n]}
                    ( (s,t'') & [\alpha_i:Gr((\Omega,\eta,\psi_i))]\sigma[t''/t][c/t''] \land\\
                    & \bigvee_{e \in [c+1, k_n]}
                    (s,t') [\alpha_i:Gr((\Omega',\eta',\psi'_i))]\sigma[e/t'] )\\
                    \bigvee_{c \in [k_1,k_n]}
                    ( (s,t'') & [\alpha_i:Gr((\Omega',\eta',\psi'_i))]\sigma[t''/t'][c/t''] \land\\
                    & \bigvee_{g \in [c+1, k_n]}
                    (s,t)[\alpha_i:Gr((\Omega,\eta,\psi_i))]\sigma[g/t] )
                  \end{align*}
                  The above expression can now be simplified as follows
                  \begin{align*}
                    & \bigvee_{c \in [k_1,k_n]}
                    (s,t'')[\alpha_i: \\
                    & \left( Gr((\Omega,\eta,\psi_i))]\sigma[t''/t][c/t''] \land
                    Gr((\Omega',\eta',\psi'_i))]\sigma[t''/t'][c/t''] \right) \\
                    & \vee 
                    \left( Gr((\Omega,\eta,\psi_i))]\sigma[t''/t][c/t'']
                    \land \bigvee_{e \in [c+1, k_n]}
                    (s,t')[\alpha_i:Gr((\Omega',\eta',\psi'_i))]\sigma[e/t']
                      \right)\\
                    & \vee
                    \left( Gr((\Omega',\eta',\psi'_i))]\sigma[t''/t'][c/t'']
                    \land \bigvee_{g \in [c+1, k_n]}
                    (s,t)[\alpha_i:Gr((\Omega,\eta,\psi_i))]\sigma[g/t] \right)]
                  \end{align*}
                  If the range of $t''$ can be forced to be $[l_1,u_1]$ and if we introduce
                  additional constraints $t''<t$ and $t''<t'$, the above grounding expression
                  is equivalent to the grounding of the following lifted explanation graph
                  \begin{align*}
                    & (\Omega\cup\{t''\}:\eta \land l_1-1<t'' \land t''-1 < u_1 \land 
                      t'' < t \land t'' < t',\\
                    & (s,t'')[ \alpha_i: \\
                    & (\psi_i[t''/t] \oplus \psi'_i[t''/t']) \lor \\
                    & (\psi_i[t''/t] \oplus \psi') \lor \\
                    & (\psi'_i[t''/t'] \oplus \psi)])
                  \end{align*}
                  Therefore the theorem is proved in this case.
                  
                  The proof for remaining cases is analogous and
                  straightforward, because based on the values of $l_1,u_1,l_2$
                  and $u_2$ we have disjuncts where the ranges of root variables
                  are either identical or non-overlapping. Both of these cases
                  have been proved already.
             \end{description}
          \end{description}
        \item[Case 2.2:] Operation $\lor$
          \begin{description}
            \item[Case 2.2.1:] When $Q(\Omega,\eta)$ is not identical to
              $Q(\Omega',\eta')$. The definition of $\lor$ operation returns
              several lifted explanation graphs. If $Q(\Omega,\eta)\sigma$ and
              $Q(\Omega',\eta')\sigma$ are both unsatisfiable then we can see
              that $Gr((\Omega:\eta,\psi)\sigma) \lor
              Gr((\Omega':\eta',\psi')\sigma) = 0$. Further, all the lifted
              explanation graphs returned by the definition of $\lor$ have
              unsatisfiable constraints therefore, the theorem is proved in this
              case. When $Q(\Omega,\eta)\sigma$ is satisfiable but not
              $Q(\Omega',\eta')\sigma$ ( or vice-versa), $Gr((\Omega:\eta,\psi)\sigma) \lor
              Gr((\Omega':\eta',\psi')\sigma) =
              Gr((\Omega:\eta,\psi)\sigma)$. Based on the definition of $\lor$
              operation, only the first lifted explanation graph has satisfiable
              constraint, so the theorem is proved. Same reasoning applies in
              the symmetric case. 
            \item[Case 2.2.2:] When $Q(\Omega,\eta)$ is identical to
              $Q(\Omega',\eta')$ and $Q(\Omega,\eta)\sigma$ is unsatisfiable,
              the proof of the theorem is trivial. So we consider the case where
              $Q(\Omega,\eta)\sigma$ and $Q(\Omega',\eta')\sigma$ are both
              satisfiable and $Q(\Omega,\eta)$ may or may not be identical to
              $Q(\Omega',\eta')$.
              \begin{description}
                \item[Case 2.2.2.1:] When $\psi=0$ (analogously $\psi'=0$). Here
                  $Gr((\Omega:\eta,\psi)\sigma) \lor
                  Gr((\Omega':\eta',\psi')\sigma) =
                  Gr((\Omega':\eta',\psi')\sigma)$. Similarly
                  $Gr((\Omega\cup\Omega':\eta\land\eta', \psi')\sigma) =
                  Gr((\Omega':\eta', \psi')\sigma)$. Therefore the theorem is
                  proved.
                \item[Case 2.2.2.2:] When $\psi=1$ (analogously $\psi'=1$). Here
                  $Gr((\Omega:\eta,\psi)\sigma) \lor
                  Gr((\Omega':\eta',\psi')\sigma) = 1$. Similarly
                  $Gr((\Omega\cup\Omega':\eta\land\eta', 1)\sigma) =
                  1$. Therefore, the theorem is proved.
                \item[Case 2.2.2.3:] Neither $\psi$ nor $\psi'$ is a leaf node
                  and $(s,t) < (s,t')$ (analogously $(s',t') < (s,t)$). Proof is
                  analogous to Case 2.1.1.3.
                \item[Case 2.2.2.4:] Neither $\psi$ nor $\psi'$ is a leaf node
                  and $(s,t)=(s',t')$. Proof is analogous to Case 2.1.1.4.
                \item[Case 2.2.2.5:] Neither $\psi$ nor $\psi'$ is a leaf node
                  and $(s,t) \not \sim (s',t')$ and $t$ is a free variable or a
                  constant and $t'$ is a free variable. Proof is analogous to
                  Case 2.1.1.5.
                \item[Case 2.2.2.6:] Neither $\psi$ nor $\psi'$ is a leaf node
                  and $(s,t) \not \sim (s',t')$ and $t$ is a free variable or a
                  constant and $t'$ is a bound variable. Proof is analogous to
                  Case 2.1.1.6.
                \item[Case 2.2.2.7:] Neither $\psi$ nor $\psi'$ is a leaf node
                  and $(s,t) \not \sim (s',t')$ and $t,t'$ are bound
                  variables. Proof is analogous to Case 2.1.1.7. However, we need
                  to show that the simplified form is valid. Again consider the RHS
                  \[
                    \bigvee_{c \in
                      [l_1,u_1]}(s,c)[\alpha_i:Gr((\Omega,\eta,\psi_i)\sigma[c/t])]
                    \land
                    \bigvee_{c' \in
                      [l_2,u_2]}(s,c')[\alpha_i:Gr((\Omega',\eta',\psi'_i)\sigma[c'/t'])]
                  \]
                  We can perform two renaming operations and rewrite the above expression
                  as
                  \begin{align*}
                    \bigvee_{c \in
                      [l_1,u_1]}(s,c) & [\alpha_i:Gr((\Omega,\eta,\psi_i)\sigma[t''/t][c/t])]
                    \land\\
                    & \bigvee_{c' \in
                      [l_2,u_2]}(s,c')[\alpha_i:Gr((\Omega',\eta',\psi'_i)\sigma[t''/t'][c'/t'])]
                  \end{align*}
                  It is straightforward to see that the above grounding is equivalent to 
                  \[
                    Gr(\Omega\cup\Omega'\cup\{t''\}\setminus\{t,t''\}: \eta\land\eta'[t''/t,t''/t'], (s,t'')[\alpha_i:\psi_i[t''/t] \lor \psi'_i[t''/t']])
                  \]
              \end{description}
          \end{description}
      \end{description}
  \end{description}
\end{proof}

\quantifycorrect*
\begin{proof}
  Let us first consider the case when the root is $(s,X)$ for some switch $s$ in
  $\psi$. $quantify((\Omega:\eta,\psi),X) = (\Omega\cup\{X\}:\eta, \psi)$. If
  $\eta\sigma_{-X}$ is unsatisfiable, then
  $Gr(quantify((\Omega:\eta,\psi),X)\sigma_{-X}) = 0$. Next,
  $\bigvee_{\sigma \in \Sigma} Gr((\Omega:\eta,\psi)\sigma[c/X]) = 0$ since
  $\eta\sigma$ is also unsatisfiable for any $\sigma$. On the other hand if
  $\eta\sigma_{-X}$ is satisfiable,
\begin{align*}
Gr(quantify((\Omega:\eta,\psi),X)\sigma_X) & =
                                           Gr((\Omega\cup\{X\}:\eta,\psi)\sigma_{-X})\\
                                         & =
                                           Gr((\Omega\cup\{X\}:\eta\sigma_{-X},\psi\sigma_{-X}))\\
                                         & = \bigvee_{c \in range(X,\eta\sigma_{-X})}
                                           (s,c)[\alpha_i:
                                           Gr(\Omega\setminus\{X\},
                                           \eta\sigma_{-X}[c/X], \psi_i\sigma_{-X}[c/X])]
\end{align*}
Next,
\[
\bigvee_{\sigma \in \Sigma} Gr((\Omega:\eta,\psi)\sigma)  =
\bigvee_{c \in \sigma_X(\eta\sigma_{-X})}
  (s,c)[\alpha_i:Gr(\Omega\setminus\{X\},\eta\sigma_{-X}[c/X],\psi_i\sigma_{-X}[c/X])]
\]
By using continuity of range,
$range(X, \eta\sigma_{-X}) = \sigma_X(\eta\sigma_{-X})$. Therefore the theorem is
proved in this case. Now consider the case where $X$ doesn't occur in the root
of the lifted explanation graph. Since, the lifted explanation graph is
well-structured, there is subtree $\psi'$ in $\psi$ such that the root of
$\psi'$ contains $X$ and all occurrences of $X$ are within $\psi'$. If we remove
the subtree $\psi'$ from $\psi$, then
$Gr(quantify((\Omega:\eta,\psi),X)\sigma_{-X}) = \bigvee_{\sigma \in \Sigma}
Gr((\Omega:\eta,\psi)\sigma)$ since all the disjuncts on the right hand
side will be identical to each other and to the ground explanation tree on the
left hand side. Therefore, we need only show that the grounding of the subtree
$\psi'$ when $X$ is a quantified variable is same as
$\bigvee_{\sigma \in \Sigma} Gr((\Omega,\eta\sigma,\psi'\sigma)$ which we
already showed.
\end{proof}

\infcorrect*
\begin{proof}
Consider the following modification of the grounding algorithm for lifted
explanation graphs. An extra argument is added to $Gr(\Omega,\eta,\psi)$ to make
it $Gr(\Omega,\eta,\psi,\sigma)$. Whenever a variable is substituted by a value
from its domain, $\sigma$ is augmented to record the substitution. Further the
set $\Omega$ and the constraint formula $\eta$ are not altered when recursively
grounding subtrees. Rather, $\eta\sigma$ is tested for satisfiability and
$t\sigma$ is tested for membership in $\Omega$ to determine if a node contains
bound variable. The grounding of a lifted explanation graph $(\Omega:\eta,\psi)$
is given by $Gr(\Omega,\eta,\psi,\{\})$. It is easy to see that the ground
explanation tree produced by this modified procedure is same as that produced by
the procedure given in definition \ref{def:grounding}. 

We will prove that if $Gr(\Omega,\eta,\psi,\sigma) = \phi$, then $prob(\phi) =
f(\sigma, \psi)$. We prove this using structural induction based on the
structure of $\psi$.
\begin{description}
  \item[Case 1:] If $\psi$ is a $0$ leaf node, then $Gr(\Omega,\eta,0, \sigma) =
    0$. Therefore $prob(\phi) = f(\sigma,\psi) = 0$.
  \item[Case 2:] If $\psi$ is a $1$ leaf node, and $\eta\sigma$ is satisfiable,
    then $Gr(\Omega,\eta,\psi,\sigma) = 1$ and $prob(\phi) = f(\sigma,\psi)$. On
    the other hand if $\eta\sigma$ is not satisfiable, then
    $Gr(\Omega,\eta,\psi,\sigma) = 0$ and $prob(\phi) = f(\sigma,\psi)$.
  \item[Case 3:] If $\psi = (s,t)[\alpha_i:\psi_i]$ and $t\sigma \not \in
    \Omega$, and $\eta\sigma$ is satisfiable, then $\phi =
    (s,t\sigma)[\alpha_i:Gr(\Omega,\eta,\psi_i,\sigma)]$. Therefore, $prob(\phi)
    = \sum_{\alpha_i \in D_s} \pi_s(\alpha_i) \cdot
    prob(Gr(\Omega,\eta,\psi_i,\sigma))$. But $f(\sigma,\psi) = \sum_{\alpha_i
      \in D_s} \pi_s(\alpha_i) \cdot f(\sigma, \psi_i)$. Therefore, by inductive
    hypothesis the theorem is proved in this case. If $\eta\sigma$ is
    unsatisfiable, then $\phi=0$, therefore $prob(\phi) = f(\sigma,\psi)$.
  \item[Case 4:] If $\psi = (s,t)[\alpha_i:\psi_i]$ and $t\sigma \in \Omega$ and
    $\eta\sigma$ is satisfiable. In this case $\phi$ is defined as the
    disjunction of the ground trees
    $(s,c)[\alpha_i:Gr(\Omega,\eta,\psi_i,\sigma[c/t])]$ where
    $c \in range(t,\eta\sigma)$. Let us order the grounding trees in the
    increasing order of the value $c$. Given two trees $\phi_{(s,c)}$ and
    $\phi_{(s,c+1)}$ corresponding to values $c,c+1 \in range(t,\eta\sigma)$,
    the $\lor$ operation on ground trees, would recursively perform disjunction
    of the $\phi_{(s,c+1)}$ with the subtrees in the following set
    \[
      Fr = \{\phi' \mid \phi' \mbox{ is a maximal subtree of } \phi_{(s,c)} \mbox{
        without } c \mbox{ as instance argument of any node}\}
    \]
    The set $Fr$ contains ground trees corresponding to the trees in
    $\id{frontier}_t(\psi)$ and possibly $0$ leaves. Since we assumed that
    frontier subsumption property is satisfied, for every $\phi' \in Fr$ that is
    not a $0$ leaf, it holds that every explanation in $\phi_{(s,c+1)}$ contains
    a subexplanation in $\phi'$. Therefore, $\phi_{(s,c)} \lor \phi_{(s,c+1)}$,
    can be computed equivalently as
    $\phi_{(s,c)} \lor (\neg \widehat{\psi}_t [c/t] \land \phi_{(s,c+1)})$.
    Since $\neg \widehat{\psi}_t [c/t]$ contains only internal nodes with
    instance argument $c$, the explanations of $\neg \widehat{\psi}_t[c/t]$ are
    independent of explanations in $\phi_{(s,c+1)}$. Further, the explanations
    of $\phi_{(s,c)}$ are mutually exclusive with explanations in
    $\neg \widehat{\psi}_t [c/t]$.  Therefore the probability
    $prob(\phi_{(s,c)} \lor \phi_{(s,c+1)})$ can be computed as
    $prob(\phi_{(s,c)}) + (1 - prob(\widehat{\psi}_t[c/t])) \cdot prob(\psi_{(s,c+1)})$.
    The probability of the complete disjunction
    $\bigvee_{c \in range(t, \eta\sigma)}
    (s,c)[\alpha_i:Gr(\Omega,\eta,\psi_i,\sigma[c/t])]$ is obtained by the
    expression
      \begin{align*}
        prob(\phi_{(s,l)}) + (1 - prob(\widehat{\psi}_t[l/t])) & \times (\\
                             prob(\phi_{(s,l+1)}) + &(1 - prob(\widehat{\psi}_t[l+1/t])) \times (\\
                                                                & \cdots\\
                                                                & (1 - prob(\widehat{\psi}_t[u-1/t])) \times prob(\phi_{(s,u)})))
      \end{align*}
      Now consider $f(\sigma,\psi)$ for the same $\psi$, $f(\sigma,\psi) =
      h(\sigma[l/t],\psi)$. The expansion of $h(\sigma[l/t],\psi)$ is as follows
      \begin{align*}
        g(\sigma[l/t],\psi) + (1 - prob(\widehat{\psi}_t[l/t])) & \times (\\
                              g(\sigma[l+1/t],\psi) + & (1 - prob(\widehat{\psi}_t[l+1/t])) \times (\\
                                                                 & \cdots \\
                                                                 & (1 - prob(\widehat{\psi}_t[u-1/t]) \times g(\sigma[u/t],\psi)
      \end{align*}
      For a given ground tree $\phi_{(s,c)}$, $prob(\phi_{(s,c)}) =
      \sum_{\alpha_i \in D_s} \pi_s(\alpha_i) \cdot
      prob(Gr(\Omega,\eta,\psi_i,\sigma[c/t]))$. Similarly $g(\sigma[c/t],\psi)
      = \sum_{\alpha_i \in D_s} \pi_s(\alpha_i) \cdot
      f(\sigma[c/t],\psi_i)$. But by inductive hypothesis,
      $Prob(Gr(\Omega,\eta,\psi_i,\sigma[c/t])) =
      f(\sigma[c/t],\psi_i)$. Therefore, $prob(\phi_{(s,c)}) =
      g(\sigma[c/t],\psi)$. Therefore $prob(\phi) = f(\sigma,\psi)$. When,
      $\eta\sigma$ is not satisfiable, $prob(\phi) = f(\sigma,\psi) =
      0$. Therefore, the theorem is proved.
\end{description}
\end{proof}

%%% Local Variables: 
%%% mode: latex
%%% TeX-master: "main"
%%% End: 

\bibliographystyle{acmtrans}
\bibliography{bibliography}

\label{lastpage}

\end{document}